%% file: main.tex
\documentclass[letterpaper]{article}
\input{packages}
\input{macros}
\graphicspath{{images/}{plots/}}

\title{Safe Policy Improvement for POMDPs via Finite-State Controllers}
\author{
    Thiago D. Sim{\~a}o,\thanks{Equal contribution.}\quad
    Marnix Suilen,\footnotemark[1] \quad
    Nils Jansen
}

\affiliations {
    Department of Software Science\\
	Radboud University\\
	Nijmegen, The Netherlands\\
    \{thiago.simao, marnix.suilen, nils.jansen\}@ru.nl
}

\begin{document}

\maketitle

\begin{abstract}

We study \emph{safe policy improvement} (SPI) for partially observable Markov decision processes (POMDPs).
SPI is an offline reinforcement learning (RL) problem that assumes access to (1) historical data about an environment, and (2) the so-called \emph{behavior policy} that previously generated this data by interacting with the environment.
SPI methods neither require access to a model nor the environment itself, and aim to reliably improve the behavior policy in an offline manner.
Existing methods make the strong assumption that the environment is fully observable.
In our novel approach to the SPI problem for POMDPs, we assume that a finite-state controller (FSC) represents the behavior policy and that finite memory is sufficient to derive optimal policies.
This assumption allows us to map the POMDP to a finite-state fully observable MDP, the \emph{history MDP}.
We estimate this MDP by combining the historical data and the memory of the FSC, and compute an improved policy using an off-the-shelf SPI algorithm.
The underlying SPI method constrains the policy-space according to the available data, such that the newly computed policy only differs from the behavior policy when sufficient data was available.
We show that this new policy, converted into a new FSC for the (unknown) POMDP, outperforms the behavior policy with high probability.
Experimental results on several well-established benchmarks show the applicability of the approach, even in cases where finite memory is not sufficient.
\end{abstract}

\section{Introduction}
Reinforcement learning (RL) is a standard approach to solve sequential decision-making problems when the environment dynamics are unknown~\citep{DBLP:books/lib/SuttonB98}. 
Typically, an RL agent interacts with the environment and optimizes its behavior according to environment's feedback.
However, in offline RL~\citep{DBLP:journals/corr/abs-2005-01643}, the RL agent receives a fixed dataset of past interactions between a behavior policy and the environment and derives a new policy with no direct interaction with the environment.
One of the challenges in offline RL is to ensure that the new policy outperforms the behavior policy \citep{DBLP:conf/icml/ChengX0A22}.
This problem is called \emph{safe policy improvement}~\cite[SPI;][]{DBLP:conf/icml/ThomasTG15}.
Most of the approaches to SPI assume fully observable environments, see for instance~\citep{Petrik2016,DBLP:conf/icml/LarocheTC19}.

\begin{figure}[tbp]
    \centering
    \resizebox{.85\columnwidth}{!}{\def\svgwidth{240pt}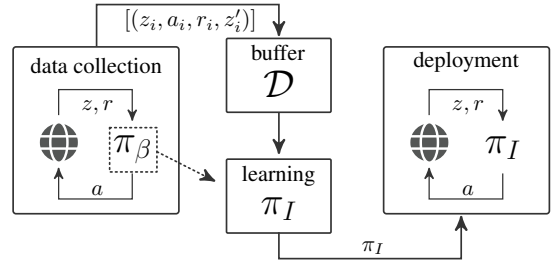}
    \caption{Illustration of the offline reinforcement learning problem in partially observable environments \citep[adapted from][]{DBLP:journals/corr/abs-2005-01643}. The dashed arrow indicates the setting where the behavior policy is available during learning.}
    \label{fig:outline}
\end{figure}

The restriction to full observability poses a serious limitation on the applicability of SPI, as most real-world problems are \emph{partially observable}, due to, for instance, noisy sensors~\citep{kochenderfer2015decision}.
\emph{Partially observable Markov decision processes} (POMDPs) are the standard model for decision-making problems under partial observability~\citep{DBLP:journals/ai/KaelblingLC98}.
So far, SPI for POMDPs was only studied for memoryless policies~\citep{DBLP:conf/aaai/ThomasTG15,DBLP:conf/itp/YeagerMNT22}.
However, POMDP policies often require a notion of memory.
In general, optimal policies for POMDPs with infinite horizons require infinite memory, rendering this problem undecidable~\cite{DBLP:journals/ai/MadaniHC03}.
Nevertheless, finite memory can make good approximations of the optimal policy \cite{DBLP:conf/icml/Bonet02} and are often used in practice for being more explainable \cite{DBLP:conf/aaai/DujardinDC17}.
Policies with finite memory may take the form of \textit{finite-state controllers}~\citep[FSCs;][]{DBLP:conf/uai/MeuleauKKC99}.

\paragraph{Our approach. }
We contribute a novel SPI approach for POMDPs.
First, to account for the inherent memory requirement in partially observable domains, we present the behavior policy as an FSC.
To create a tractable method, 
we assume that there exists a finite-memory policy for the POMDP that is optimal, also known as the {finite-history-window} approach~\citep[Section 7.3]{DBLP:journals/jair/KaelblingLM96}.
This assumption allows us to cast the POMDP as an equivalent, fully observable, \emph{history MDP} that is finite, instead of the standard infinite-history MDP~\citep{DBLP:conf/nips/SilverV10}.
We are then able to reliably estimate the transition and reward models of this finite-history MDP from the available data.
We employ a specific SPI method for MDPs, called safe policy improvement with baseline bootstrapping~\cite[SPIBB;][]{DBLP:conf/icml/LarocheTC19}.
In particular, we compute an improved policy that outperforms the behavior policy up to an \emph{admissible performance loss} with high probability.
In comparison to the approach for mere MDPs~\citep{DBLP:conf/icml/LarocheTC19}, we derive an improved bound on this admissible performance loss by exploiting the specific structure of the history MDP.
The general flow of our approach is illustrated in Figure~\ref{fig:outline}.

\paragraph{Real-world applications.}
This setting captures multiple real-world applications, such as predictive maintenance, conservation of endangered species, and management of invasive species. We may, for instance, have data from the degradation process of a certain asset, which includes logs of inspections and maintenance that were performed according to a fixed schedule (represented, for instance, as a finite-state controller). Once we acquire a new asset, we can formalize the optimization problem with offline RL to compute a new schedule, using the original schedule as a behavior policy.

We demonstrate the applicability of our method on three standard POMDP problems.
The evaluation confirms the theoretical findings of our SPIBB approach, in comparison to standard offline RL.
We highlight results for varying sizes of memory, and show that we can achieve reliable performance improvement even for problems where finite memory is not sufficient in general.

\section{Background}\label{sec:background}
For a countable set $X$ we write $|X|$ for the number of elements in $X$, and $\Dist(X)$ for the set of probability distributions over $X$ with finite support.
Given two probability distributions  $P, Q \in \Dist(X)$, the \emph{L1-distance} between $P$ and $Q$ is
\[
    \| P - Q\|_1= \sum_{x\in X}|P(x) - Q(x)|.
\]
The \emph{L1-error} of an estimated probability distribution $\tilde{P}$ is given by the L1-distance between $\tilde{P}$ and the true distribution $P$: $\| \tilde{P} - P\|_1$.
Finally, we write $\mathbb{I}(x=y)$ for the indicator function returning $1$ when $x = y$ and $0$ otherwise, and $[l:m]$ for the set of natural numbers $\{l, \dots, m\} \subset \mathbb{N}$.

\subsection{MDPs, POMDPs, and FSCs}
\begin{definition}[POMDP]
A \emph{partially observable Markov decision process} (POMDP) is a tuple $\mathcal{M} = \langle S, A, T, R, \gamma, Z, O \rangle$, where $S$, $A$ and $Z$ are finite sets of states, actions, and observations, $T \colon S \times A \to \Dist(S)$ is the transition function, $R \colon S \times A \to [R_{\min} , R_{\max}] \subset \R$ is the reward function with known bounds, $\gamma \in [0,1) \subset \R$ is the discount factor, and $O \colon S \times A \to \Dist(Z)$ is the observation function.
\end{definition}
As a special case we have the (fully observable) Markov decision process \citep[MDP;][]{Puterman1994}, where $Z = S$ and $O(z | s, a) = \mathbb{I}(z=s)$, so it can be defined as a POMDP without observations:  $M = \langle S,A,T,R,\gamma \rangle$.

A \emph{history} is a sequence of observations and actions: $h \in (Z \times A)^* \times Z$.
We denote the set of all histories by $\sH$, and $\sH_k$ denotes all histories of maximal length $k$, where the length $|h|$ is the number of observations in the history $h$.

A \emph{belief} $b \in \Dist(S)$ is a distribution over the states of a POMDP.
Beliefs are \emph{sufficient statistics} for histories in POMDPs~\citep{aastrom1965optimal,DBLP:journals/ior/SmallwoodS73}.
That is, they provide just as much information as the histories themselves.
A belief $b$ can be updated into a new belief~$b'$ upon taking an action $a$ and receiving an observation $z$ by performing a Bayesian \emph{belief update}~\cite{DBLP:journals/ai/KaelblingLC98}.
The belief $b'$ of being in state $s$ given some history $h'$, denoted $b(s \mid h')$, can be recursively computed by repeated applications of the belief update on actions $a$ and observations $z$ in the history $h' = haz$:
\[
 b'(s' \mid haz) = b'(s' \mid b(\cdot \mid h), a, z).
\]
until the history $h$ is empty, denoted $\emptyset$, and where $b(\cdot \mid \emptyset)$ is the initial belief.

A POMDP is equivalent to an infinite-state fully observable MDP called the \emph{history MDP}~\citep{DBLP:conf/nips/SilverV10}.
\begin{definition}[History MDP]\label{def:hismdp}
The fully observable \emph{history MDP} of a POMDP is $\langle \sH, A, T_{\sH}, R_{\sH}, \gamma \rangle$, where $\sH$ is the set of all histories, $A$ and $\gamma$ are the actions and discount factor from the POMDP, and $T_\sH \colon \sH \times A \to \Dist(\sH)$ and $R_\sH \colon \sH \times A \to \R$ are the transition and reward functions:
\begin{align*}
  &  T_\sH(haz \mid h, a) = \sum_{s \in S} b(s \mid h) \sum_{s' \in S} T(s' \mid s, a) O(z \mid s', a),\\
  &  R_\sH(h,a) = \sum_{s \in S} b(s \mid h) R(s,a).
\end{align*}
\end{definition}
Since belief states are sufficient statistics for histories, the so-called \emph{belief MDP} of a POMDP serves as a common alternative for the history MDP, we refer to to~\citet{DBLP:journals/ai/KaelblingLC98} for more details.

A \emph{policy} for a POMDP is a function $\pi \colon \sH \to \Dist(A)$, mapping histories to distributions over actions.
The set of all policies is $\Pi$.
The goal is to compute a policy that maximizes the expected discounted reward in the infinite horizon:
\[
\max_{\pi \in \Pi} \E\left[ \sum_{t=1}^\infty \gamma^t r_t \right],
\]
where $r_t$ is the reward the agent receives at time step $t$ when following policy $\pi$.
In general, a policy that maximizes the expected discounted reward requires infinite memory, that is, it needs to account for all possible histories.
As such, computing optimal policies in POMDPs is undecidable in general~\citep{DBLP:journals/ai/MadaniHC03}.

We may instead use policies that only use a finite amount of memory.
In general, such policies are not optimal, that is, they do not maximize the expected discounted reward, yet they are computationally more tractable.
A \emph{finite-memory} policy maps finite histories to actions, $\pi \colon \sH_k \to \Dist(A)$, and can be represented by a finite-state controller of size~$\kappa = (|Z|+1)^{k}$~\citep{DBLP:conf/uai/MeuleauKKC99,DBLP:conf/uai/Junges0WQWK018}, as we need to account for all possible combinations of observations of size $k$, with a possible empty observation.

\begin{definition}[Finite-state controller]
A \emph{finite-state controller} (FSC) is a tuple $\langle \sN, n^0, \actionmap, \memupdate \rangle$ where
$\sN$ is a finite set of \emph{memory nodes}, $n^0 \in \sN$ is an \emph{initial node}, $\actionmap \colon \sN \times Z \to \Dist(A)$ is an \emph{action mapping}, and $\memupdate \colon \sN \times Z \times A \to \sN$ is a \emph{memory update function}.
A $\kappa$-FSC is an FSC with $|\sN| = \kappa$.
\end{definition}

In any time step $t$, given the current memory node $n_t$ and observation $z_t$, the action $a_t$ is randomly drawn from the distribution $\actionmap(\cdot \mid n_t,z_t)$, and the memory node is updated to $n_{t+1} = \memupdate(n_t,z_t,a_t)$.
A policy represented by an FSC can get arbitrarily close to the optimal policy as $\kappa$ grows~\citep{DBLP:conf/icml/Bonet02}.
Computing finite-state controllers that aim to maximize the expected (discounted) reward can be done in several ways, among those gradient descent~\citep{DBLP:conf/uai/MeuleauPKK99} and convex optimization~\citep{DBLP:journals/aamas/AmatoBZ10,DBLP:conf/uai/Junges0WQWK018,DBLP:conf/aaai/Cubuktepe0JMST21}.

We denote the set of finite-memory policies of size $k$ by~$\Pi_k$, and the set of finite-memory policies of size $k$ represented by FSCs with some fixed memory update $\memupdate$ by $\Pi_k^\memupdate$.
Furthermore, FSCs of size $k$ can be represented by FSCs with more memory.
These sets are related by the following inclusions, where $k' > k$:
$
    \Pi_k^\memupdate \subset \Pi_k \subset \Pi_{k'} \subset \Pi.
$

Finally, we define the state-based and state-action-based value functions on an (PO)MDP $M$ with policy $\pi$ as $V_\pi^M(s)$ and $Q_\pi^M(s,a)$ respectively.
Whenever $M$ or $\pi$ are clear from the context we will omit them.
The \emph{performance} of a policy $\pi$ in $M$ is denoted by $\rho(\pi, M)$ and is defined as the expected value in some initial state $s_0$, that is, $\rho(\pi, M) = V_\pi^M(s_0)$. 
Furthermore, we write $V_{\max}$ for a known upper bound on the absolute value of the performance: $V_{\max} \leq \nicefrac{R_{\max}}{1-\gamma}$.

\subsection{Safe Policy Improvement on MDPs}
Here, we review safe policy improvement (SPI) for MDPs.
A \emph{dataset} is a sequence $\D$ of trajectories collected under a \emph{behavior policy} $\pib$ in an MDP $M^*$. 
For MDPs, the datasets we consider for SPI are of the form $\D = \langle s_t, a_t, r_t \rangle_{t \in [1:m]}$, and we write $\#_{\D}(x)$ for the number of times $x$ occurs in $\D$.
The goal of SPI is to compute a new policy $\pii$ based on $\D$ that outperforms $\pib$ with an allowed performance loss~$\zeta \in \R$ with high probability~$1-\delta$.

SPI operates on a set of \emph{admissible MDPs} $\Xi$, that is, MDPs $M = (S,A,T,R,\gamma)$ which are `close' to an MDP $\mlemdp = (S,A,\tilde{T}, \tilde{R}, \gamma)$ estimated from a dataset $\D$ by maximum likelihood estimation (MLE).

\begin{definition}[MLE-MDP]\label{def:mlemdp}
The MLE-MDP of an unknown true MDP $M^* = (S,A, T, R, \gamma)$ and a dataset $\D$ is a tuple  $\mlemdp = (S,A,\tilde{T}, \tilde{R}, \gamma)$ with 
transition and reward functions $\tilde{T}$ and $\tilde{R}$ derived from $\D$ via maximum likelihood estimation:
\[
\tilde{T}(s' \mid s, a) = \frac{\#_{\D}(s,a,s')}{\#_{\D}(s,a)}, \text{ and }
    \tilde{R}(s,a) = \frac{R_{\text{total}}(s,a)}{\#_{\D}(s,a)},
\]
where $R_{\text{total}}(s,a)$ is the sum of all rewards in $(s,a)$:
$R_{\text{total}}(s,a) = \sum_{(s_t, a_t, r_t) \in \D} \mathbb{I}(s_t=s \wedge a_t = a) \cdot r_t$.
\end{definition}
For an MLE-MDP $\mlemdp$ and error function $e \colon S \times A \to \R$, we define the set of admissible MDPs $\Xi_e^{\mlemdp}$ with a transition function $T$ that has $L_1$ distance to the estimated transition function $\tilde{T}$ bounded by the error function $e$:
\begin{align*}
    &\Xi_e^{\mlemdp} = \{ M \mid \forall (s,a). \| T(\cdot \mid s,a) - \tilde{T}(\cdot \mid s,a) \|_1 \leq e(s,a) \}.
\end{align*}

The general idea behind SPI methods is to define this error function $e$ such that $\Xi_e^{\mlemdp}$ includes the true MDP $M^*$ with high probability $1-\delta$ \citep[Proposition~$9$]{Petrik2016}.
Then one can compute a new policy which is an improvement \textit{for all} MDPs within $\Xi_e^{\mlemdp}$.
An alternative is to simply solve the MLE-MDP, but this could lead to arbitrarily poor policies when the amount of data is insufficient.
If, however, the amount of data is sufficient for all state-action pairs, then we can guarantee with high probability that the improved policy computed on the MLE-MDP has a higher performance.
Specifically, as pointed by \citet{DBLP:conf/icml/LarocheTC19}, the amount of data is sufficient when for all state-action pairs
\begin{equation}
\#_{\D}(s,a) \geq \frac{8V_{\max}^2}{\zeta^2 (1-\gamma)^2} \log \frac{2|S||A|2^{|S|}}{\delta}.\label{eq:spi_n_hat}
\end{equation}
Then with probability $1-\delta$ an optimal policy $\pii$ for $\mlemdp$ is $\zeta$-approximately safe with respect to the true MDP $M^*$ for some \emph{admissible performance loss} $\zeta \in \R$.
That is,
\[
    \rho(\pii, M^*) \geq \rho(\pi^*, M^*) - \zeta \geq \rho(\pib, M^*) - \zeta,
\]
where $\pi^*$ is an optimal policy in the true MDP $M^*$.
Intuitively this ensures that the estimated transition function is close enough to the true MDP to guarantee that the policy computed in the MLE-MDP approximately outperforms the behavior policy in the underlying MDP.

\subsection{SPI with Baseline Bootstrapping on MDPs}
The bound in Equation~\eqref{eq:spi_n_hat} needs to hold for every state-action pair, which limits the practical use of optimizing the MLE-MDP.
The \emph{SPI with baseline bootstrapping} \citep[SPIBB;][]{DBLP:conf/icml/LarocheTC19} algorithm overcomes this limitation, allowing the constraint in \eqref{eq:spi_n_hat} to be violated on some state-action pairs.
These state-action pairs are collected in $\sU$, the set of \emph{unknown} state-action pairs with counts smaller than a given hyperparameter $\Nmin$:
\[
    \sU = \{ (s,a) \in S \times A \mid \#_{\D}(s,a) \leq \Nmin \}.
\]
SPIBB computes an improved policy $\pii$ on $\mlemdp$ as above, except that $\pii$ is constrained to follow $\pib$ for unknown state-action pairs: $\forall (s,a) \in \sU.\, \pii(a \mid s) = \pib(a \mid s)$.
Then, $\pii$ is a $\zeta$-approximately safe improvement of $\pib$ with high probability $1-\delta$,
where $\zeta$ is computed via~\citep[Theorem 2;][]{DBLP:conf/icml/LarocheTC19}:
\begin{equation}
\begin{split}\small
\zeta = \frac{4V_{\max}}{1-\gamma} \sqrt{\frac{2}{\Nmin} \log \frac{2|S||A|2^{|S|}}{\delta}} - \rho(\pii, \mlemdp) + \rho(\pib, \mlemdp).\label{eq:spibb_zeta_bound}
\end{split}
\end{equation}

\section{SPIBB for POMDPs}
Now we detail our approach to apply SPIBB to POMDPs.

\paragraph{Formal problem statement. } 
Given a POMDP $\pomdp = \langle S, A, T, R, \gamma, Z, O \rangle$ of which the transition and observation functions are unknown, some initial belief $b \in \Dist(S)$, and a finite-memory behavior policy represented as a $\kappa$-FSC $\pib = (\sN, n^0, \actionmap, \memupdate)$, the goal is to apply SPIBB to construct a new $\kappa$-FSC $\pii = (\sN,n^0,\actionmap', \memupdate)$ with the same nodes and memory structure $\memupdate$, \ie{} $\pib, \pii \in \Pi_k^\memupdate$, such that with high probability $1-\delta$, $\pii$ is a $\zeta$-approximately safe improvement over $\pib$ with respect to $\pomdp$.
That is, with a probability of at least $1-\delta$ we have
\[
\rho(\pii, \pomdp) \geq \rho(\pib, \pomdp) - \zeta.
\]

\subsection{From POMDP to Finite-History MDP}
While a POMDP can be mapped to a fully observable history MDP (Definition~\ref{def:hismdp}), this MDP has infinitely many states, making a direct application of SPI(BB) methods infeasible.
To mitigate this issue, we make an assumption on the structure of the history MDP (and inherently on the POMDP) that implies that the history MDP is equivalent to a smaller, finite, MDP. 
We formalize this assumption via stochastic bisimulation~\citep{DBLP:journals/ai/GivanDG03}.
Intuitively, this bisimulation is an equivalence relation that relates (history) states that \emph{behave} similarly according to reward signals.

\begin{definition}[Bisimilarity of history states]
A stochastic bisimulation relation $E \subseteq \sH \times \sH$ on history states $h_1, h_2 \in \sH$ is an equivalence relation satisfying
\begin{align*}
&E(h_1,h_2) \iff \forall a\in A.\quad R_\sH(h_1,a) = R_\sH(h_2,a) \text{ and } \\
 &\quad \forall h_1',h_2' \in \sH \text{ with } E(h_1',h_2') \text{ we have }\\
 &\quad  \qquad T_\sH(h_1' \mid h_1, a) = T_\sH(h_2' \mid h_2, a).
\end{align*}
The largest stochastic bisimulation relation is called (stochastic) bisimulation, denoted by $\sim$.
We write $[h]_\sim$ for the equivalence class of history $h$ under $\sim$, and $\sH/_{\sim}$ for the set of equivalence classes.
\end{definition}

\begin{assumption}[Sufficiency of finite histories]\label{asm:1}
Every history state $h$ of size $|h| > k$ in the history MDP is bisimilar to a history state $h'$ of size $|h'| \leq k$.
That is, $h \sim h'$.
\end{assumption}
\noindent As a consequence, the history MDP satisfying Assumption~\ref{asm:1} has a finite bisimulation quotient MDP~\cite{DBLP:journals/ai/GivanDG03}, and we call it a \emph{finite-history MDP} instead.
This finite-history MDP consists of states that are the equivalence classes of histories under $\sim$.

\begin{definition}[Finite-history MDP]\label{def:finhismdp}
A POMDP satisfying Assumption 1 is a fully observable finite-state MDP $M = (\sH/_{\sim}, A, T_H, R_H, \gamma)$
where the states are given by the set of equivalence classes, the actions and discount factor from the POMDP, and transition and reward functions defined as
\begin{align*}
      &T_H([haz]_\sim \mid [h]_\sim, a) = \\
      &\quad\sum_{s \in S} b(s \mid [h]_\sim) \sum_{s' \in S} T(s' \mid s, a)O(z \mid s', a),\\
      &R_H([h]_\sim, a) = \sum_{s \in S} b(s \mid [h]_\sim)R(s,a).
\end{align*}
\end{definition}

\noindent Under bisimulation equivalence, the finite-history MDP and the POMDP are related in the following fundamental way.
\begin{theorem}[Optimal finite-memory policies under bisimilarity]\label{thm:bisim}
An optimal policy $\pi^*$ in the finite-history MDP is an optimal finite-memory policy for the POMDP.
\end{theorem}

Theorem~\ref{thm:bisim} is a direct result of bisimilarity~\cite{DBLP:journals/ai/GivanDG03}.
We may number the equivalence classes in the finite-history MDP in such a way that they correspond to memory nodes of an FSC.
As a result, the finite-history MDP can be defined on a state-space consisting of memory nodes and observations, rather than histories.

\begin{definition}[Finite-history MDP via FSC]\label{def:fscmdp}
A POMDP satisfying Assumption 1 is a fully observable finite-state MDP $M = (\sN \times Z, A, T_H, R_H, \gamma)$
where the states are given by pairs of memory nodes from an FSC and observations, the actions from the POMDP, and transition and reward functions defined as
\begin{align*}
      &T_H(\nzprime \mid \nz, a) = \\
      &\,\,\sum_{s \in S} b(s \mid \nz) \sum_{s' \in S} T(s' \mid s, a)O(z' \mid s', a)\memupdate(n' \mid n, z', a),\\
      &R_H(\nza) = \sum_{s \in S} b(s \mid \nz)R(s,a).
\end{align*}
Where $b(s \mid \nz)$ is the belief of being in state $s$ of the POMDP given memory node $n$ and observation $z$.
\end{definition}
This finite-history MDP will serve as the (unknown) true MDP $M^*$ in our application of SPIBB, and $\memupdate$ is the memory update function of the FSC.

\subsection{Estimating the Finite-History MDP}
Next, we describe how to estimate the true finite-history MDP $M^*$ by an MLE-MDP $\mlemdp$.
The approach is similar to that of SPI for MDPs described in Section~\ref{sec:background}, except that the dataset $\D$ is different.
Here, $\D$ is collected from simulating the POMDP $\pomdp$ under (FSC) policy $\pib$.
This yields a dataset of the form
\begin{equation}\label{eq:dataset}
    \D = \langle \langle n_{t},z_{t} \rangle, a_{t}, r_t \rangle_{t \in [1:m]}
\end{equation}
where the observations $z_t$ come from the observation function, and the memory nodes $n_t$ are observed from the FSC.

\begin{definition}[Finite-history MLE-MDP]\label{def:finhismlemdp}
The MDP from Definition~\ref{def:fscmdp} can be estimated from a dataset $\D$ of the form~\eqref{eq:dataset}, following the same approach for estimating a standard MLE-MDP as in Definition~\ref{def:mlemdp}:
\begin{align*}
  &  \tilde{T}_H( \nzprime \mid \nza) = \frac{\#_{\D}(\nza, \nzprime)}{\#_{\D}(\nza)}, \text{ and } \\
  &  \tilde{R}_H(\nza) = \frac{R_\text{total}(\nza)}{\#_{\D}(\nza)}, \text{ where }\\
    &R_\text{total}(\nza) = \hspace{-2em}
    \sum\limits_{(\langle n_{t},z_{t} \rangle, a_t, r_t) \in \D} \hspace{-2em} \mathbb{I}(\langle n_{t},z_{t} \rangle = \langle n, z \rangle \wedge a_t = a) \cdot r_t.
\end{align*}

\end{definition}

\subsection{Applying SPIBB to the Finite-History MDP}
In this section we apply the theory of SPIBB, as introduced in Section~\ref{sec:background}, to our setting.
In particular, we have just defined a true MDP $M^*$ (the finite-history MDP, Definition~\ref{def:finhismdp}) and an MLE-MDP $\mlemdp$ estimating $M^*$ (Definition~\ref{def:finhismlemdp}).
Let 
$$
\sU = \{ (\nza) \in \sN \times Z \times A \mid \#_{\D}(\nza) \leq \Nmin \}
$$ 
be the set of tuples $(\nza)$ which occur less than $\Nmin$ times in the dataset $\D$ for some hyperparameter $\Nmin$. 
Just as in SPIBB for MDPs, we compute a new policy $\pii \in \Pi_{k}^\memupdate$ for the MLE-MDP $\mlemdp$ that estimates the finite-history MDP, constrained to follow the behavior policy $\pib$ used to collect $\D$ for all $(\nza) \in \sU$.

\begin{theorem}[$\zeta$-bound on history MDP]\label{th:spibb_history_mdp}
Let $\Pib$ be the set of policies under the constraint of following $\pib$ when $(\nza) \in \sU$.
Then, the policy $\pii$ computed by the SPIBB algorithm on the history MDP (Definition~\ref{def:hismdp}) is a $\zeta$-approximate safe policy improvement over the behavior policy $\pib$ with high probability $1-\delta$,
where:
\begin{equation*}
    \zeta =
    \frac{4V_{\max}}{1-\gamma} \sqrt{\frac{2}{\Nmin}\log \frac{2 |\sH||A|2^{|Z|}}{\delta}} - \rho(\pii,\mlemdp)  + \rho(\pib,\mlemdp).
\end{equation*}
\end{theorem}
The proof replaces the regular MDP from the SPIBB algorithm by the (infinite) history MDP. 
We can reduce the exponent from $|\sH|$, which would be the result of naively applying the SPIBB algorithm, to $|Z|$ because of the structure of the transition function of the history MDP.
In particular, the transition function of the history MDP is defined for histories $h$ which are appended by an action $a$ and an observation~$z$ to $haz$, see Definition~\ref{def:hismdp}.
As such, the successor states of $h$ in the history MDP are fully determined by the observation~$z$ instead of the full state-space, and thus we may replace $2^{|S|}$ from Equation~\eqref{eq:spi_n_hat} by $2^{|Z|}$.
The full proof can be found in Appendix~A.

While Theorem~\ref{th:spibb_history_mdp} and its proof reason over the full history MDP,
these results extend to the finite-history MDP when Assumption~\ref{asm:1} is satisfied.
We have the following corollary.

\begin{corollary}[$\zeta$-bound on finite-history MDP]\label{thm:zeta_finhis_mdp}
Let $\Pib$ be the set of policies under the constraint of following $\pib$ when
$(\nza) \in \sU$.
Then, the policy $\pii$ computed by the SPIBB algorithm in the finite-history MDP $M^*$ of a POMDP satisfying Assumption~\ref{asm:1} is a $\zeta$-approximate safe policy improvement over the behavior policy $\pib$ with high probability $1-\delta$,
where the admissible performance loss $\zeta$ is given by
\begin{align*}
    &\zeta =
    \frac{4V_{\max}}{1-\gamma} \sqrt{\frac{2}{\Nmin}\log \frac{2 |\NZ||A|2^{|Z|}}{\delta}} 
    \\ & \qquad \qquad \qquad \qquad \qquad - \rho(\pii,\mlemdp) + \rho(\pib,\mlemdp).
\end{align*}
\end{corollary}
Since bisimilarity is an equivalence relation, the finite-history MDP is equivalent to the full history MDP, and thus also the POMDP, see Theorem~\ref{thm:bisim}. 
As a consequence, the proof of Corollary~\ref{thm:zeta_finhis_mdp} follows immediately from Theorem~\ref{th:spibb_history_mdp} and the fact that bisimulation is an equivalence relation.

\begin{figure}[tbp]
    \centering
    \includegraphics[width=.33\columnwidth]{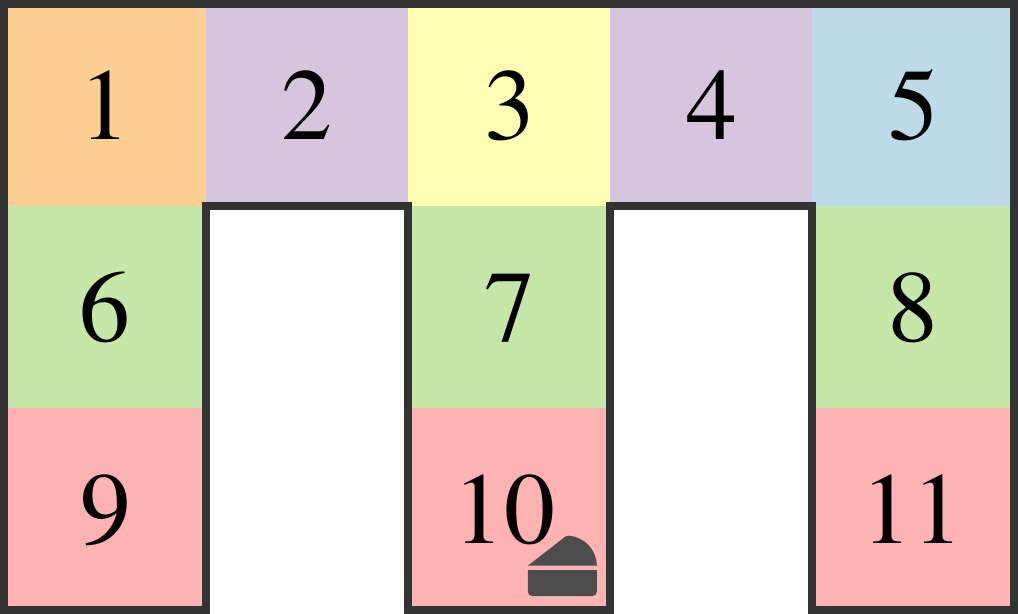}
    \caption{The Maze environment. The locations are colored according to the agent's perception.}
    \label{fig:cheese}
\end{figure}

\section{Empirical Analysis}

This section contains the empirical evaluation of our approach to SPI on POMDPs. 
We first describe the setup of the experiments, and then present and analyze the results.
We provide further details in Appendix B and code at \texttt{https://github.com/LAVA-LAB/spi\_pomdp}.

\begin{figure*}[t]
    \centering
    \begin{minipage}{.87\textwidth}
    \includegraphics[width=.24\textwidth]{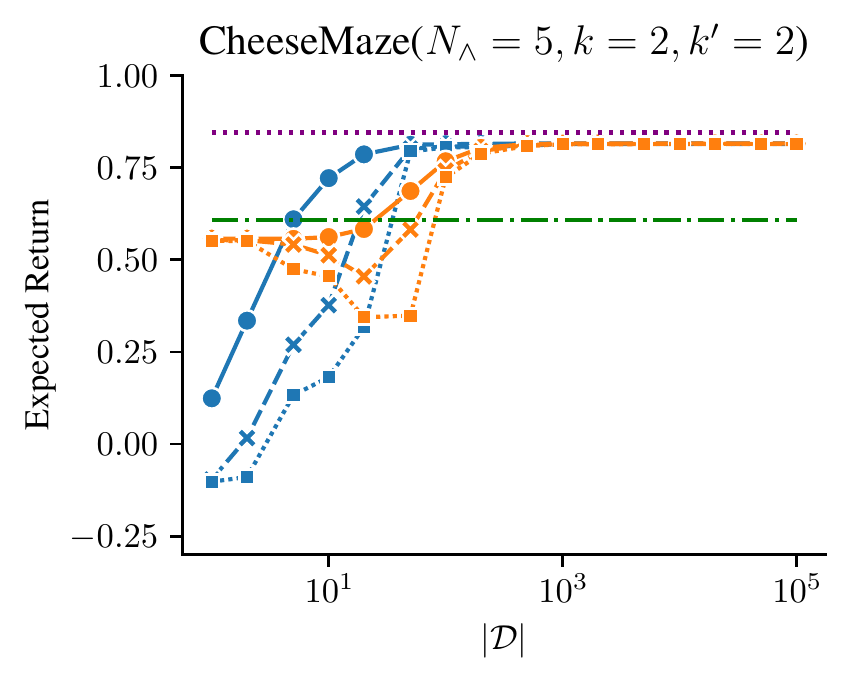} \includegraphics[width=.24\textwidth]{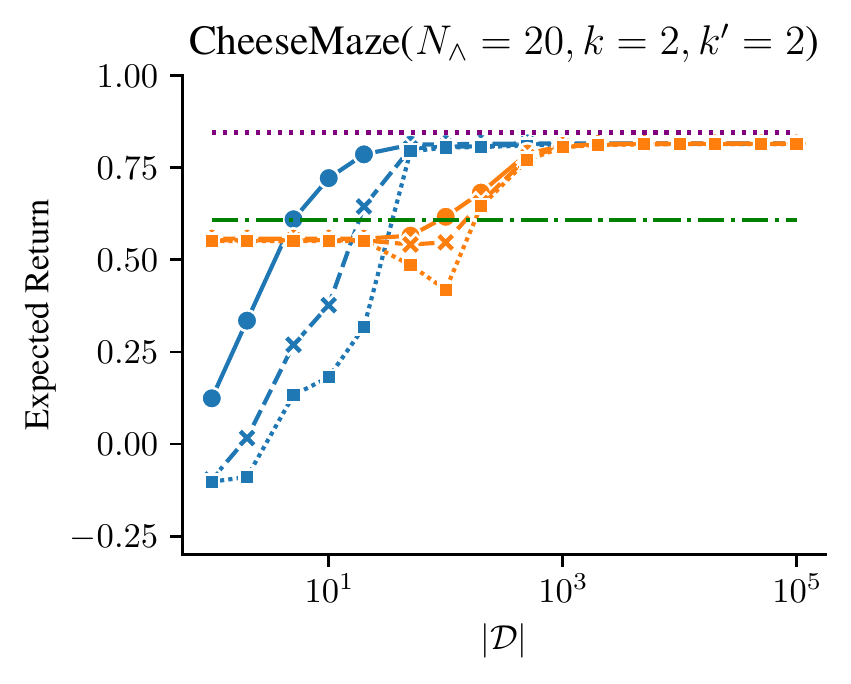}
    \includegraphics[width=.24\textwidth]{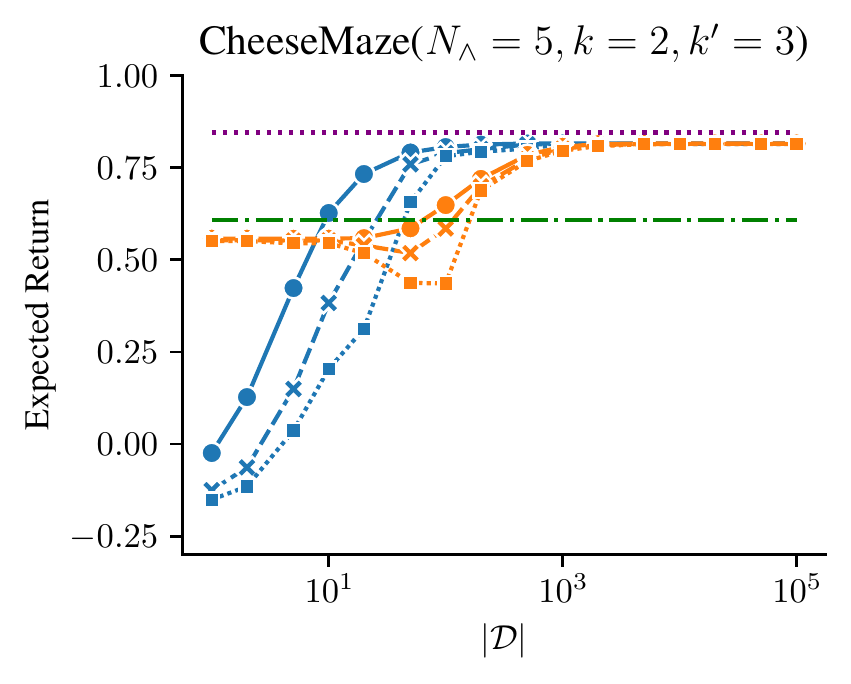}
    \includegraphics[width=.24\textwidth]{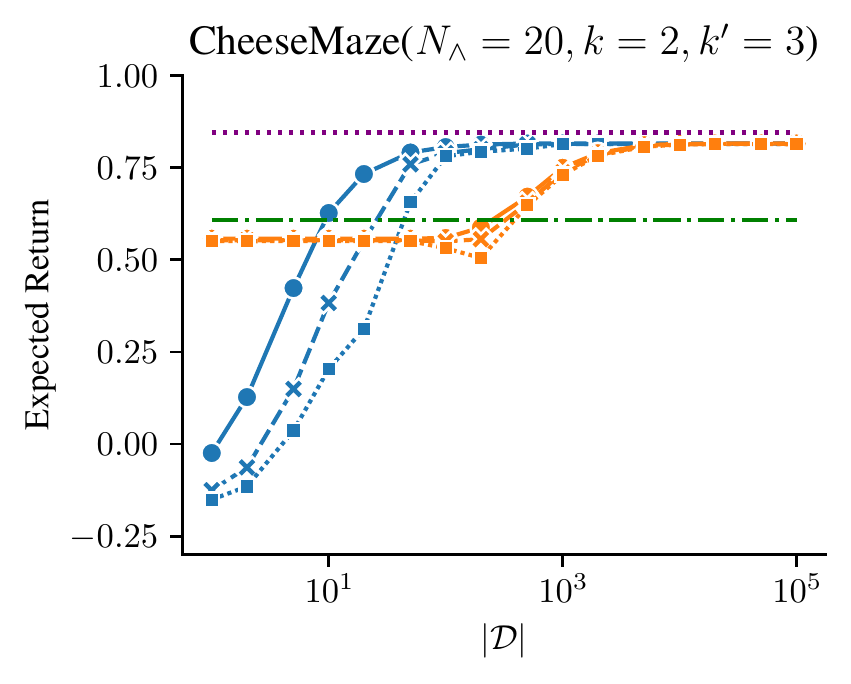}\\
    \includegraphics[width=.24\textwidth]{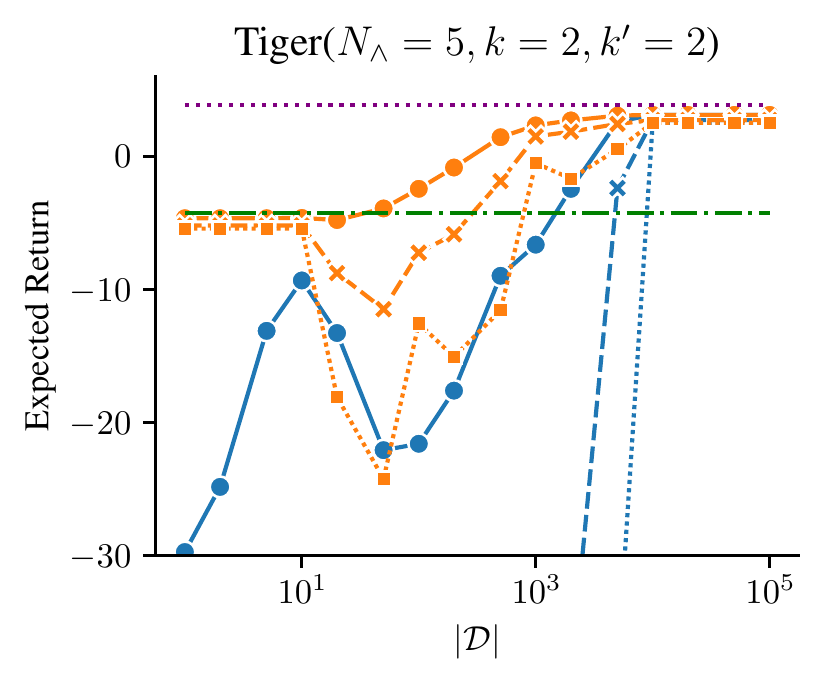} 
    \includegraphics[width=.24\textwidth]{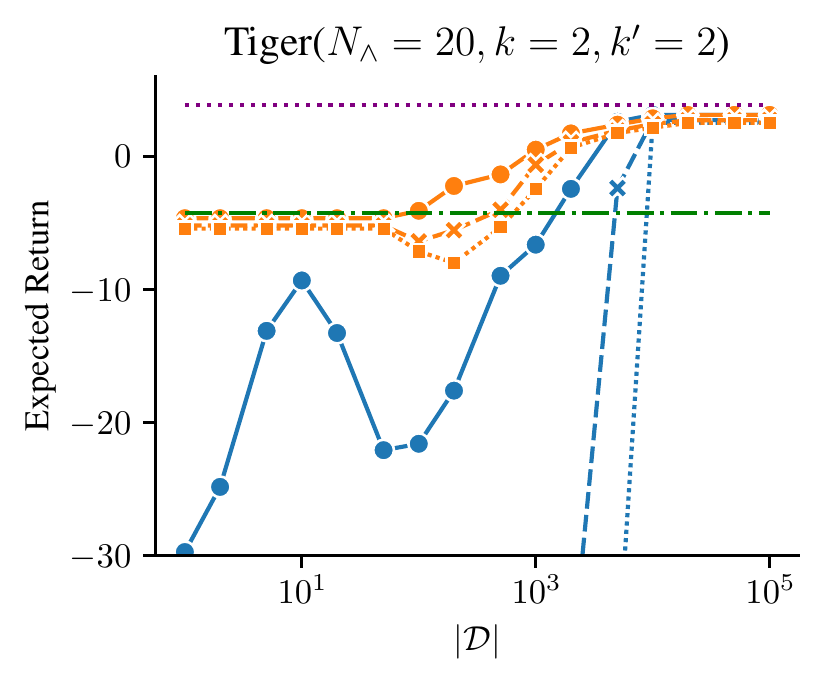}
    \includegraphics[width=.24\textwidth]{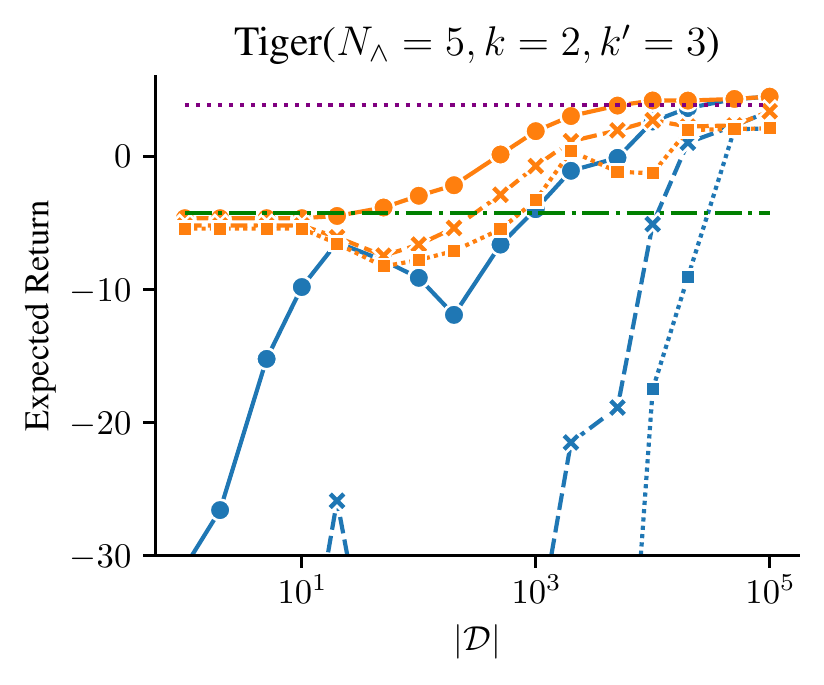} 
    \includegraphics[width=.24\textwidth]{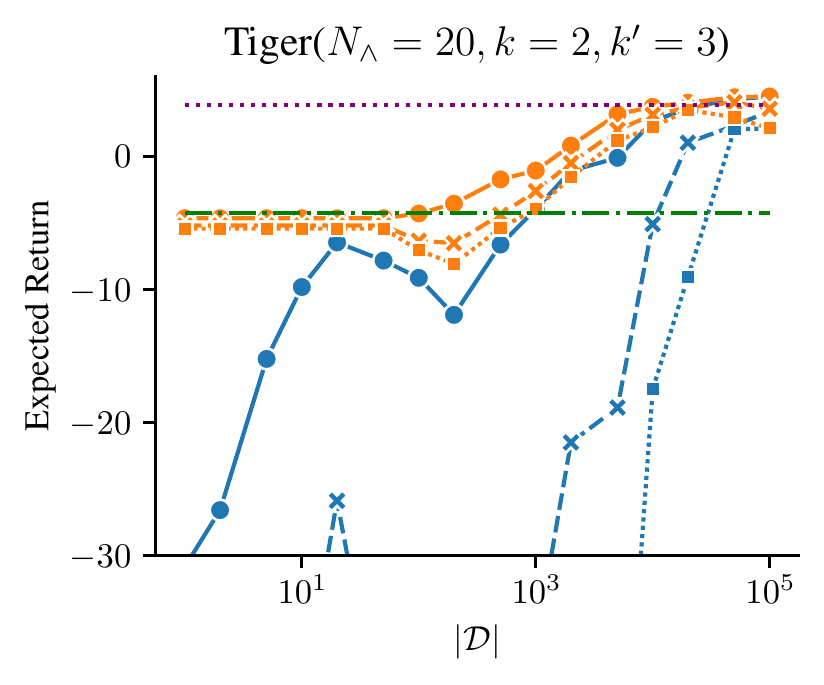} \\
    \includegraphics[width=.24\textwidth]{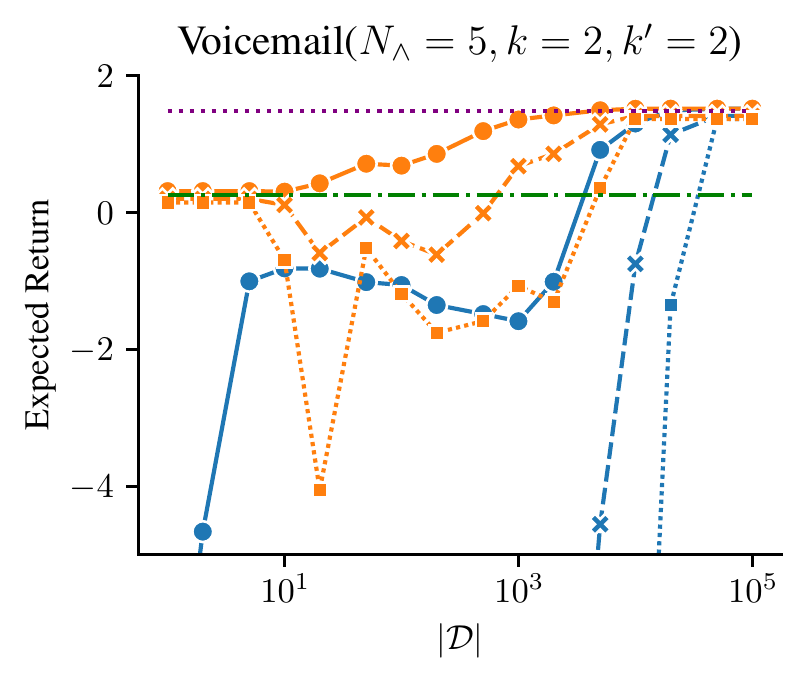}  
    \includegraphics[width=.24\textwidth]{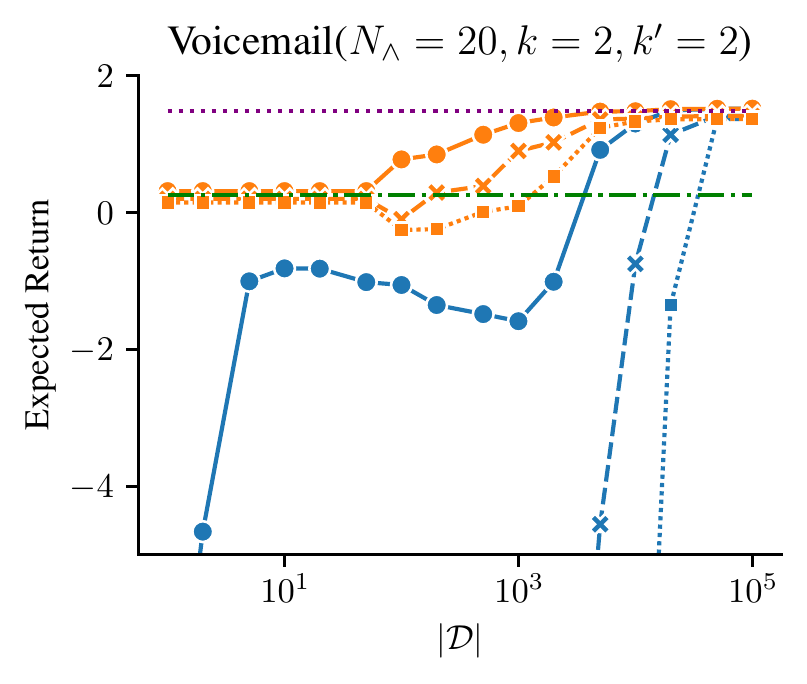} 
    \includegraphics[width=.24\textwidth]{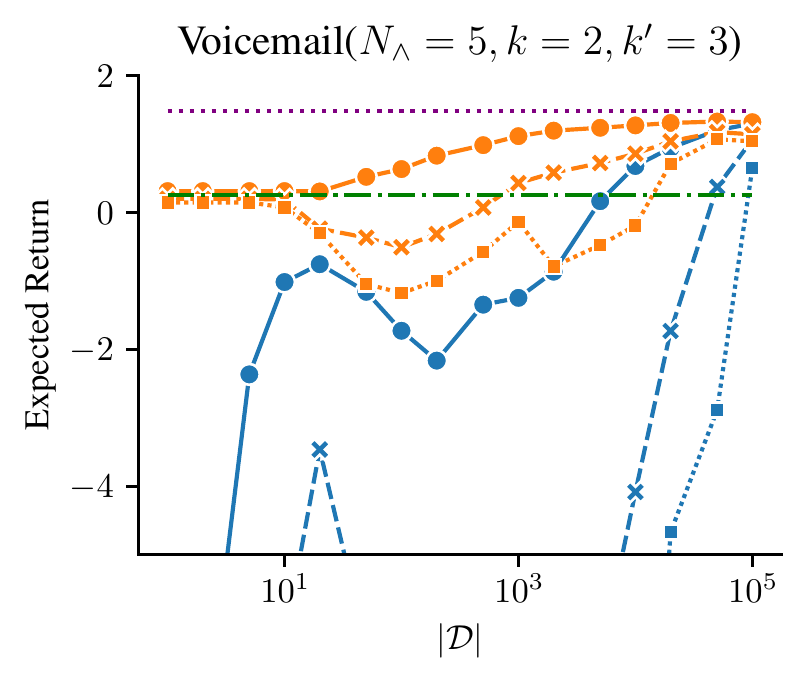}
    \includegraphics[width=.24\textwidth]{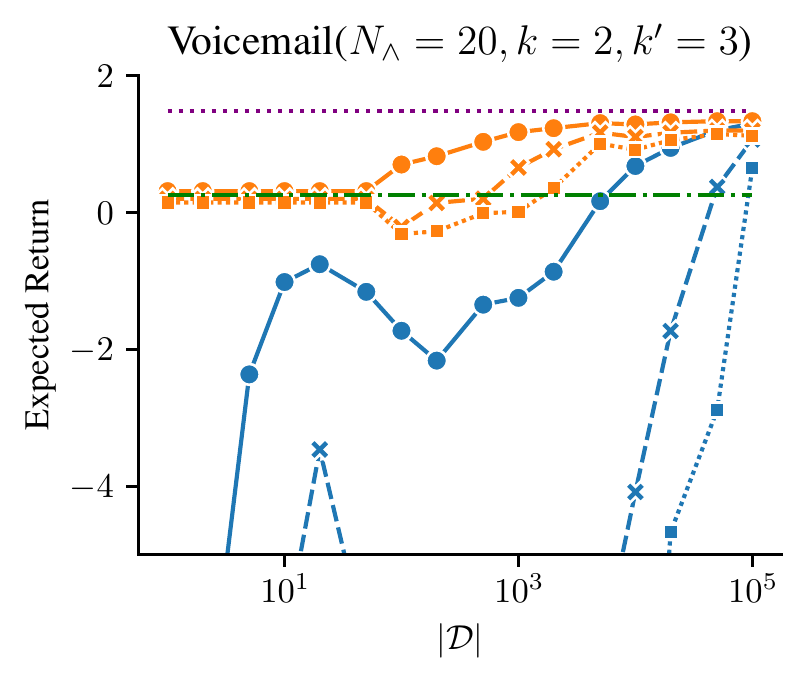}
    \end{minipage}
    \begin{minipage}{.12\textwidth}
    \includegraphics[width=\textwidth]{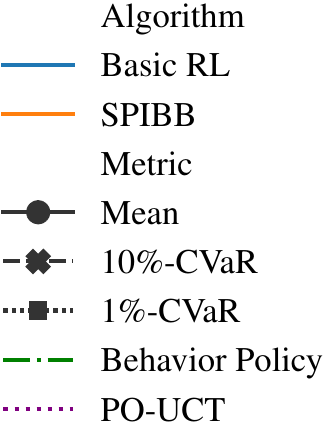} 
    \end{minipage}
    \caption{Policy improvement on the environments Maze, Tiger and Voicemail (first, second and third row, respectively) for datasets collected by a behavior policy with history size $k=2$, varying the hyperparameters pairs column-wise: $(\Nmin = 5, k' = 2)$, $(\Nmin = 20, k' = 2)$,  $(\Nmin = 5, k' = 3)$, and $(\Nmin = 20, k' = 3)$.
    The plots show the mean (solid line), $10\%$-CVaR (dashed line) and $1\%$-CVaR (dotted line). 
    The performance of the behavior policy is shown in green (dash-dotted line).}
    \label{fig:performance}
\end{figure*}

\subsection{Setup}
\paragraph{Environments.}
We consider three POMDP problems:
\begin{enumerate}
    \item CheeseMaze~\citep{DBLP:conf/icml/McCallum93}: An agent navigates a maze, moving %
    in the four cardinal directions, 
    but in each state it only perceives whether or not there is a barrier in each direction (see Figure~\ref{fig:cheese}).
    The agent is placed at a random location at the beginning of an episode, and receives a positive reward~($+1$) if it reaches the goal~(\faIcon{cheese}), and a small negative reward ($-0.01$) otherwise.
    The episode ends when the agent reaches the goal.
    \item Tiger~\citep{DBLP:journals/ai/KaelblingLC98}: An agent is in front of two doors, and a tiger is randomly positioned behind one of them at the beginning of each episode.
    The agent has three actions: 
    Listening, or opening one of the doors.
    Listening gives a noisy observation of the position of the tiger, and a small negative reward~($-1$).
    Opening the door with the tiger gives a large negative reward~($-100$), while opening the other door gives a positive reward~($+10$).
    \item Voicemail~\citep{DBLP:journals/csl/WilliamsY07}:
    An agent controls a voicemail machine,
    at the beginning of the episode the user listens to a message and decides if they want to keep it. 
    This information is hidden from the agent, which
    has three actions: \emph{ask}, \emph{save}, and \emph{delete}.
    Asking the user if they want to keep the message gives the agent a small negative reward~($-1$) and a noisy observation of the user's intention.
    Correctly saving the message gives a positive reward~($+5$), and a negative reward~($-10$) otherwise.
    Correctly deleting the message gives a positive reward~($+5$), and a negative reward~($-20$) otherwise.
\end{enumerate}
\paragraph{Satisfaction of Assumption~\ref{asm:1}.} Note that the Maze environment is close to satisfying Assumption~\ref{asm:1} for memory that looks back two steps, \ie, $k=2$, with the exceptions of histories with equal observations.
Tiger and Voicemail do not satisfy the assumption for any $k$.

\paragraph{Behavior policies.}
The behavior policies are generated via Q-learning using the memory of an FSC that keeps track of the last $k \in \{1,2\}$ observations as the state.
After convergence, we extract a softmax policy, to ensure the behavior policy explores different actions during data collection.

\paragraph{Data collection.}
We consider datasets of different sizes, namely: $1$, $2$, $5$, $10$, $20$, $50$, $\cdots$, $5\,000$, and $10\,000$ trajectories.
For each environment, number of trajectories, and behavior policy, we generate $500$ datasets.

\paragraph{Learning.}
We consider two algorithms to compute a new policy: SPIBB, and Basic RL.
Both algorithms operate on the finite-history MLE-MDP~(Definition~\ref{def:finhismlemdp}) related to the finite-history MDP of the POMDP.
We implement Basic RL as an unconstrained SPIBB where $\Nmin = 0$, that is, it solves the MLE-MDP using value iteration.
For each dataset, we compute new policies $\pii$ using each offline RL algorithm, considering different hyperparameters: $\Nmin \in \{5,7,10,15,20,30,50,70,100\}$ and $k' \in \{k, k+1\}$, where $k'$ is the history size encoded in the FSC of $\pii$.

\paragraph{Evaluation metrics.}
Each policy is evaluated over $10\,000$ episodes to obtain an estimate of the performance of the improved policy, $\rho(\pii, M^*)$.
For an evaluation across environments and behavior policies, we also consider the normalized policy improvement:
$$
    \bar{\rho}(\pii) = \frac{\rho(\pii, M^*) - \rho(\pib, M^*)}{\rho(\pi_{\max}, M^*) - \rho(\pib, M^*)},
$$
where $\pi_{\max}$ is the policy with the highest expected return in each environment.
To aggregate the results across the $500$ repetitions, we compute the mean and Conditional Value at Risk~\citep[CVaR;][]{Rockafellar2000}.
We use $x\%$-CVaR to indicate the mean of the $x\%$ lowest performances.
As an approximation of the optimal value, we show the performance of PO-UCT \citep{DBLP:conf/nips/SilverV10}, which uses the environment as a simulator to compute a policy.

\begin{figure*}[tbp]
    \centering
    \begin{minipage}{.87\textwidth}
    \includegraphics[width=.24\textwidth]{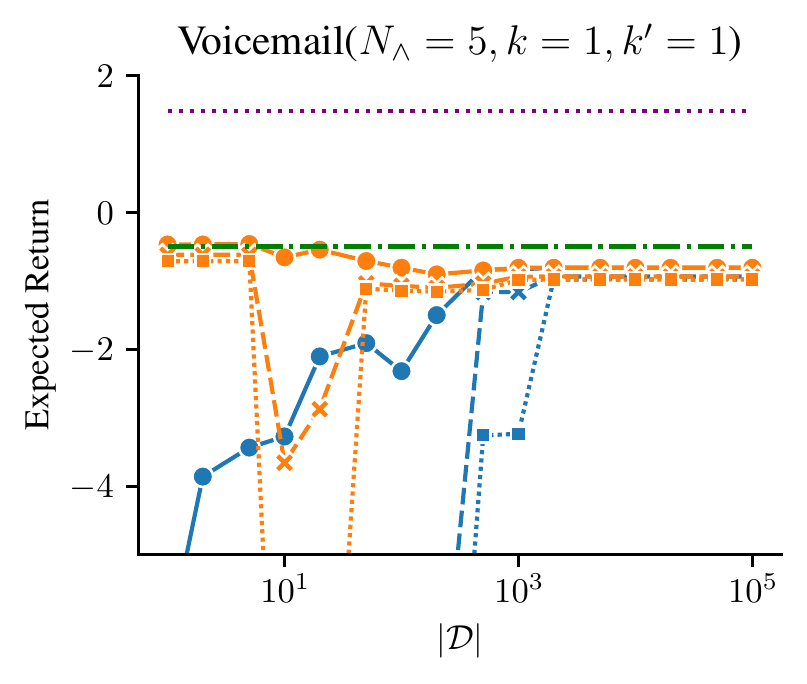}  
    \includegraphics[width=.24\textwidth]{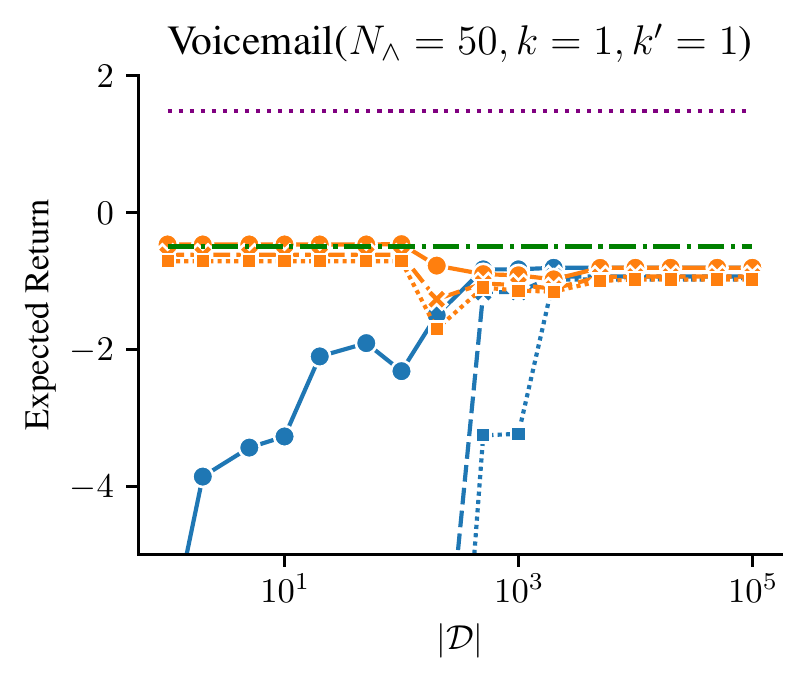} 
    \includegraphics[width=.24\textwidth]{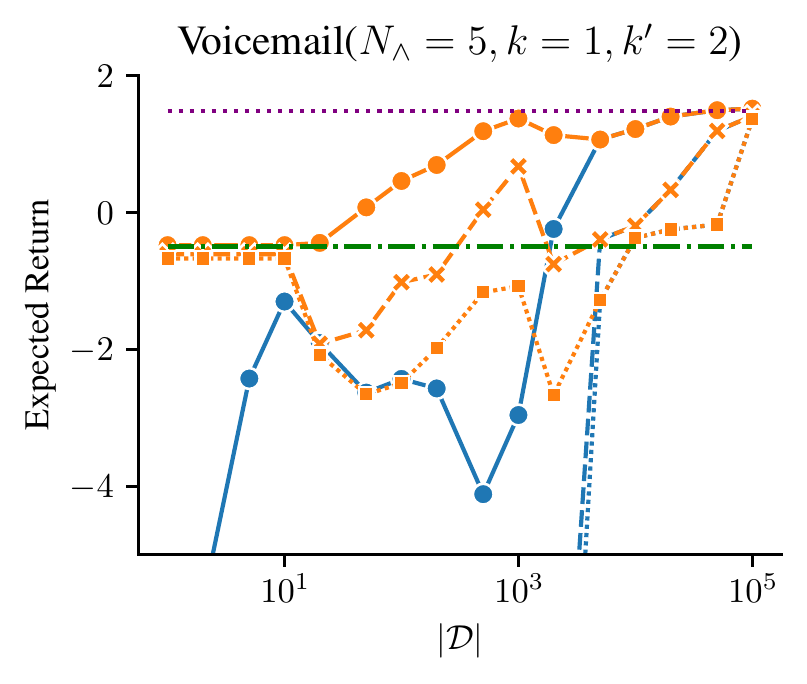} 
    \includegraphics[width=.24\textwidth]{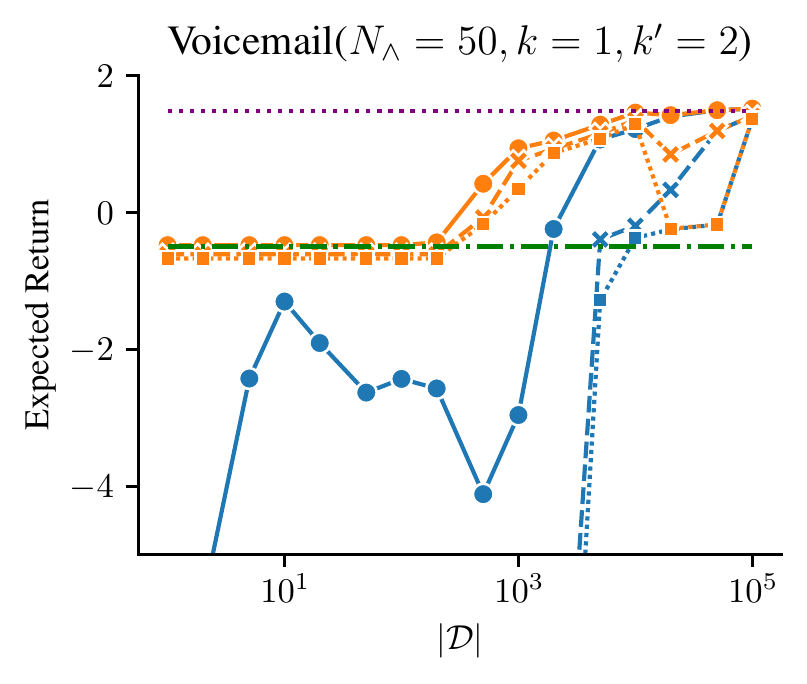}
    \end{minipage}
    \begin{minipage}{.12\textwidth}
    \includegraphics[width=\textwidth]{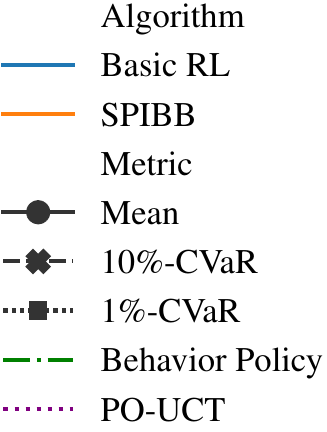} 
    \end{minipage}
    \caption{Policy improvement on the Voicemail environment for datasets collected with a memoryless policy ($k=1$), varying the hyperparameters pairs column-wise: $(\Nmin = 5, k' = 1)$, $(\Nmin = 50, k' = 1)$,  $(\Nmin = 5, k' = 2)$, and $(\Nmin = 50, k' = 2)$. The plots show the mean (solid line), $10\%$-CVaR (dashed line) and $1\%$-CVaR (dotted line). The performance of the behavior policy is shown in green (dash-dotted line).}
    \label{fig:performance_voicemail_memoryless}
\end{figure*}

\subsection{Results}
Figure~\ref{fig:performance} shows results on the three environments (ordered by row).
The data was collected using a behavior policy with $k=2$.
The first column shows the results where SPIBB uses a low threshold to consider a history-action pair known and the same memory size as the behavior policy ($\Nmin = 5$ and $k' = 2$).
The second column shows the results with a higher threshold ($\Nmin=20$ and $k'=2$).
The third column shows the results for increased memory ($\Nmin=5$ and $k'=3$).
Finally, the fourth column shows the results with a higher threshold and increased memory ($\Nmin=20$ and $k'=3$).
Basic RL is included everywhere to give a perspective on the influence of different hyperparameters.

Figures~\ref{fig:performance_voicemail_memoryless} and~\ref{fig:tiger_heatmap_memoryless} extend the empirical analysis on the Voicemail and Tiger environments for memoryless behavior policy ($k=1$), since they demonstrated to be more challenging for the safe policy improvement problem.
Figure~\ref{fig:performance_voicemail_memoryless} considers the Voicemail environment, while
Figure~\ref{fig:tiger_heatmap_memoryless} shows the normalized results for a range of thresholds in the Tiger environment.
We provide further results in the Appendix C.

\subsection{Analysis}

\paragraph{Basic RL is unreliable.}
Across all environments, the Basic RL algorithm shows a considerable performance drop compared to the behavior policy, even in terms of the mean performance for smaller datasets.
Notice that for Tiger and Voicemail, the CVaR metrics are often outside the graph.

\paragraph{SPIBB outperforms Basic RL.}
In the environments Tiger and Voicemail (Figure~\ref{fig:performance}, second and third row), the SPIBB algorithm shows better performance than the Basic RL across all dataset sizes.
This is likely due to the SPIBB algorithm retaining the randomization of the behavior policy when insufficient data is available.

\paragraph{SPIBB is reliable when Assumption~\ref{asm:1} is satisfied.}
Analyzing the results for the Maze environment (Figure~\ref{fig:performance}, first row), we observe that SPIBB shows reliably outperforms the behavior policy even for a small $\Nmin$ (first column), for which only the $1\%$-CVaR shows a performance drop.

\newcommand{\heatmapdescription}{
The left, middle and right columns show the mean, $10\%$-CVaR and $1\%$-CVaR, respectively.
The first row shows the results where the improved policy uses the same memory as the behavior policy ($k' = k$), while the second row shows the results for an improved policy with more memory ($k' = k + 1$).
}

\begin{figure}
    \centering
    \begin{minipage}{.92\columnwidth}
    \includegraphics[width=.32\textwidth]{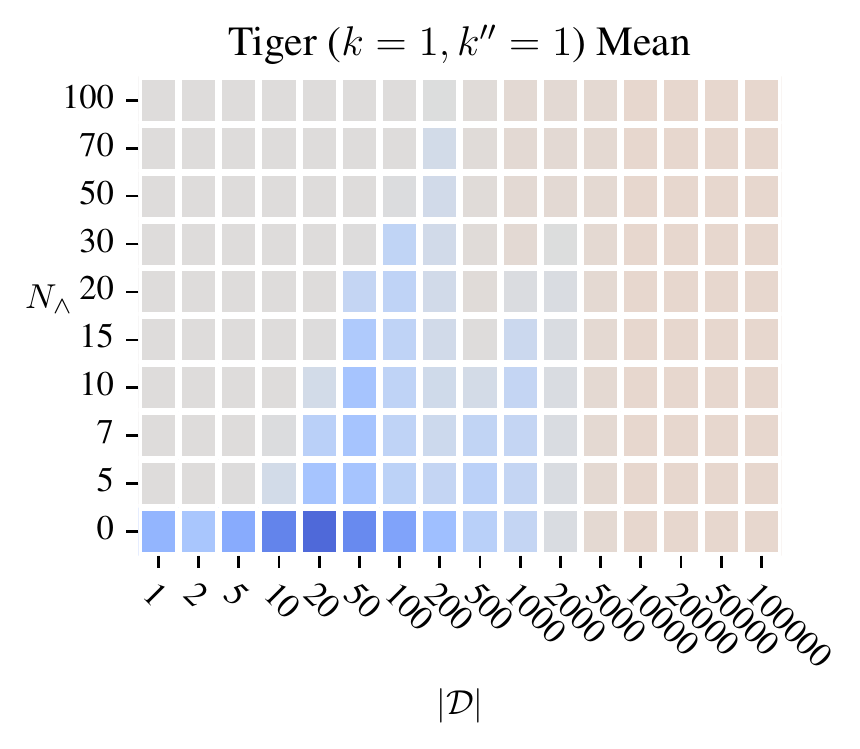}   
    \includegraphics[width=.32\textwidth]{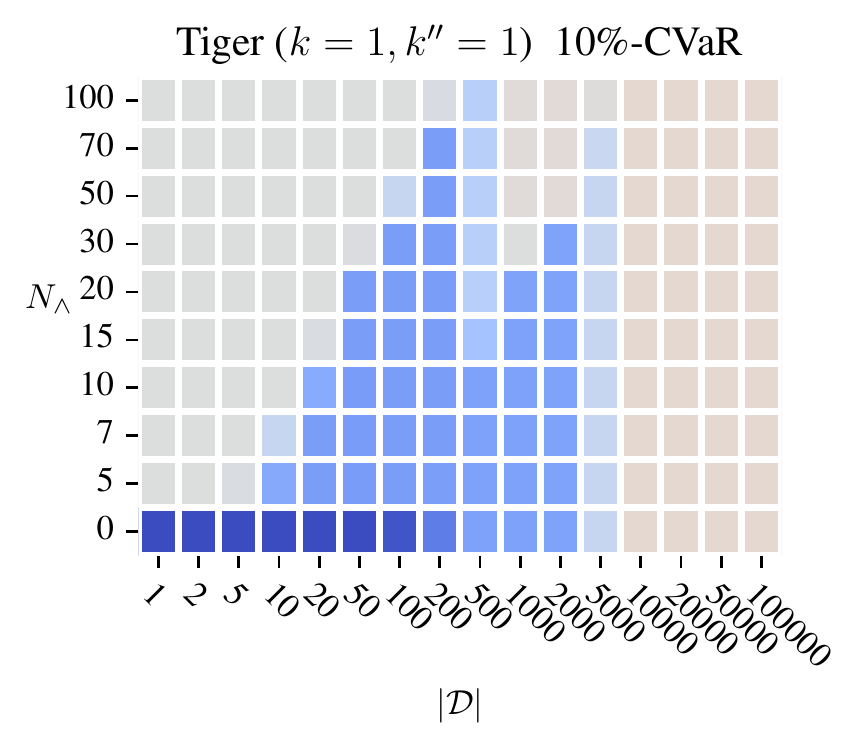} \includegraphics[width=.32\textwidth]{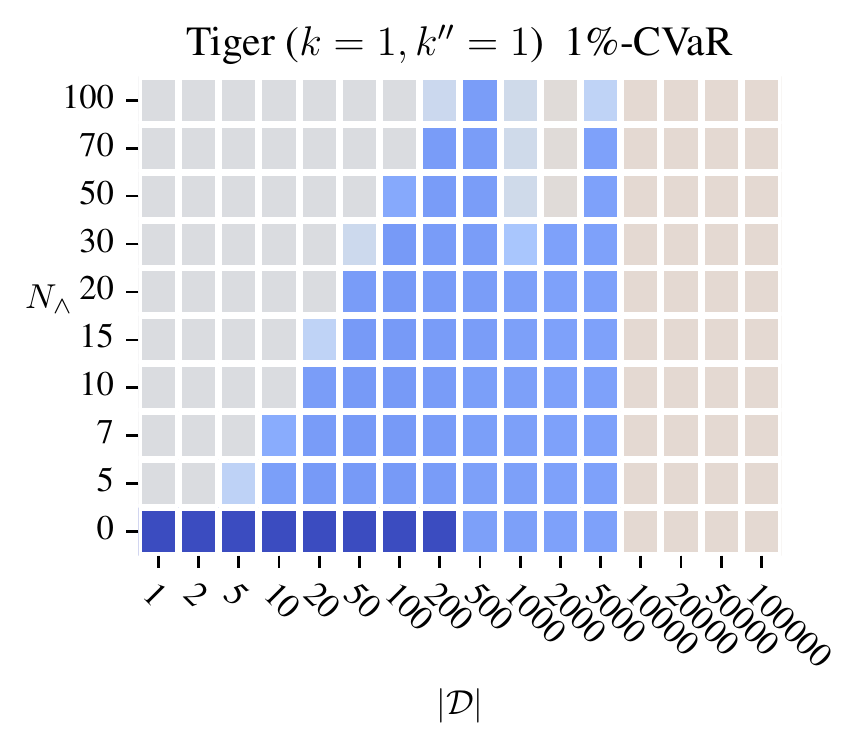}\\
    \includegraphics[width=.32\textwidth]{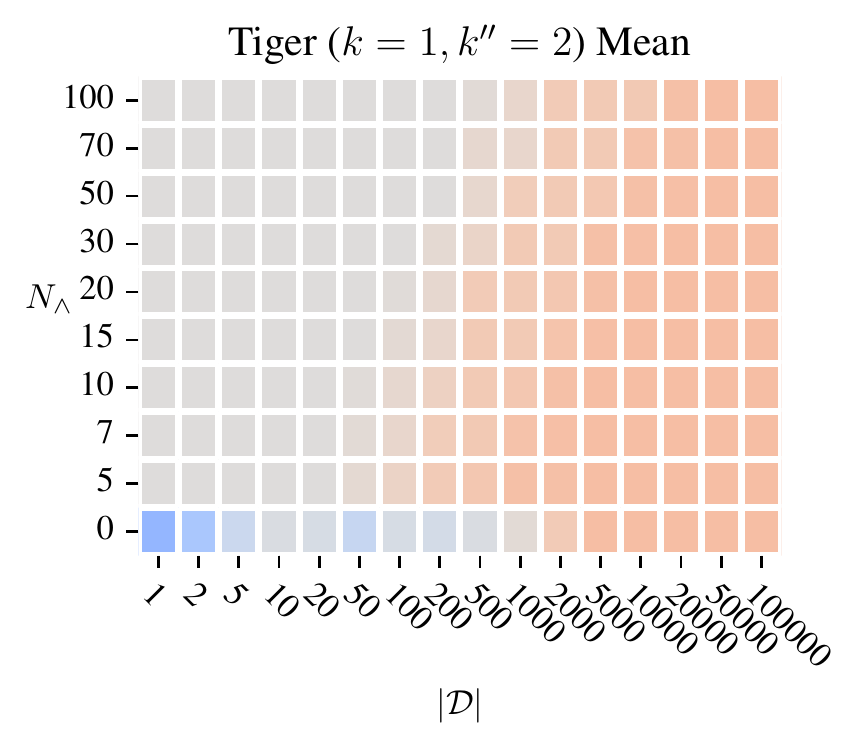}
    \includegraphics[width=.32\textwidth]{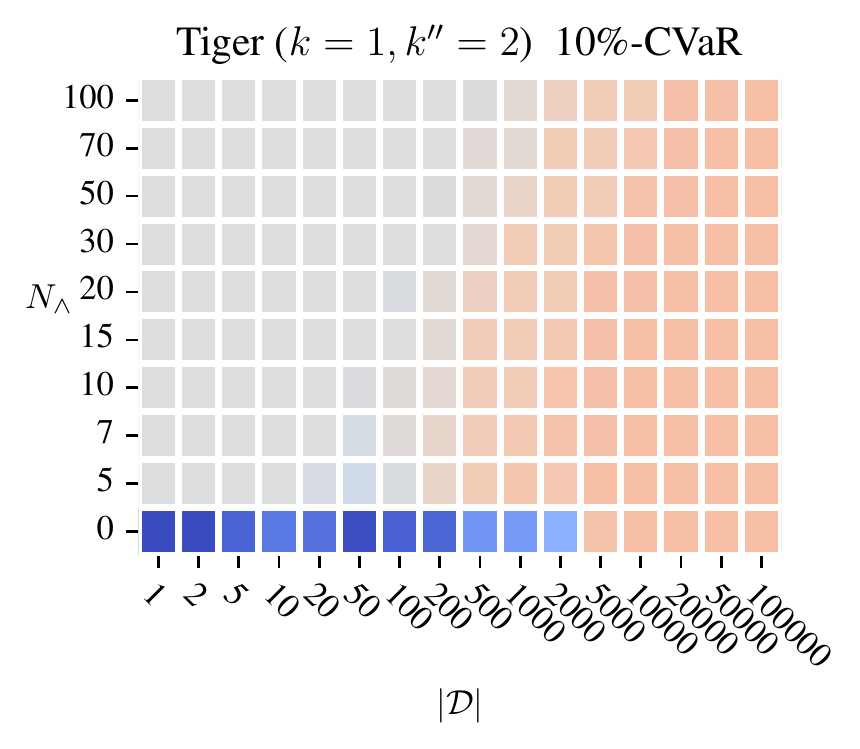}
    \includegraphics[width=.32\textwidth]{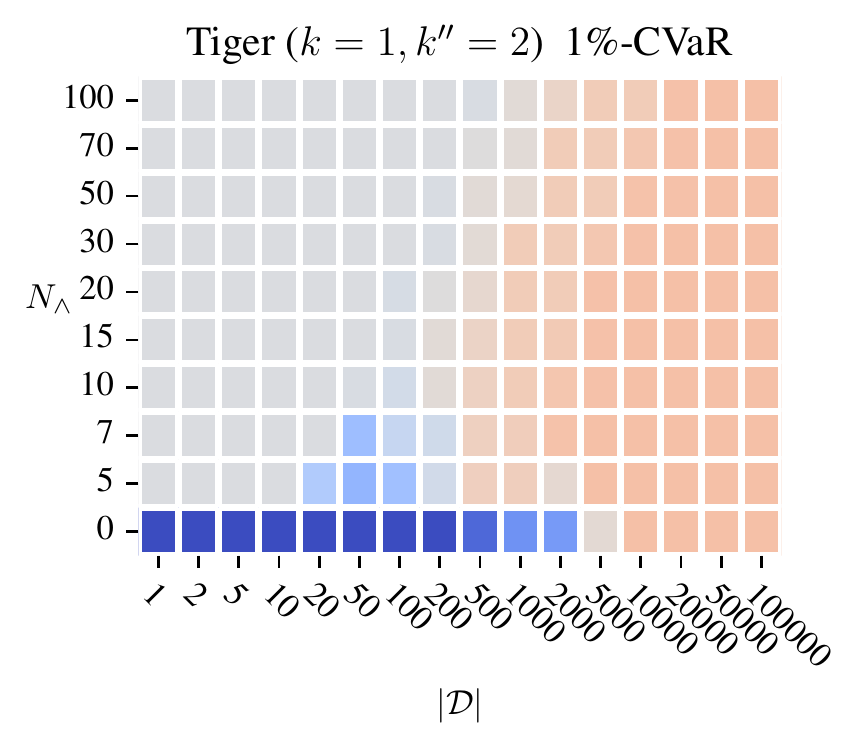}
    \end{minipage}
    \begin{minipage}{.07\columnwidth}
    \includegraphics[width=\textwidth]{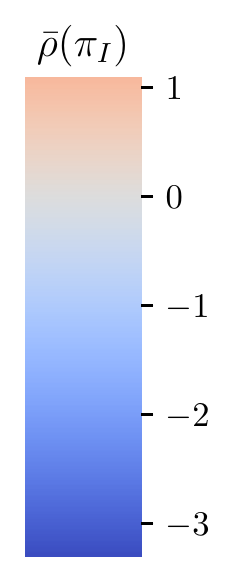}
    \end{minipage}
    \caption{Normalized performance $\bar{\rho}(\pii)$ on the Tiger environment ($k=1$). \heatmapdescription{}
    }
    \label{fig:tiger_heatmap_memoryless}
\end{figure}

\paragraph{More memory improves the reliability.}
SPIBB shows slightly unreliable behavior for small values of $\Nmin$ in the Tiger and Voicemail environments (Figure~\ref{fig:performance}), as evidenced by both the CVaR curves, which can be alleviated by increasing the $\Nmin$ or the memory of the new policy (second, third and fourth column).
When Assumption~\ref{asm:1} is violated, the performance drop may be significant, as seen in the first two columns of Figure~\ref{fig:performance_voicemail_memoryless}.
In this case, merely increasing the $\Nmin$ threshold is not enough to guarantee a policy improvement.
Increasing the memory size, however, allows the SPIBB algorithm to reliably improve the behavior policy, as Figure~\ref{fig:performance_voicemail_memoryless} (last column) and Figure~\ref{fig:tiger_heatmap_memoryless} (second row) show.

\paragraph{Deterministic policies may require more memory.}
Figure~\ref{fig:performance_voicemail_memoryless} shows an interesting phenomenon.
In partially observable settings, the stochastic behavior policy might perform better than the new deterministic policy, since randomization can trade-off some amount of memory.
We observe that when $k=1$, SPIBB and Basic RL converge to deterministic policies with an expected return lower than the behavior policy.
When SPIBB has sufficient data, it is not constrained to follow the behavior policy, and thus does not inherit any randomization from that policy.
As stated in the previous paragraph, more memory can then yield a new deterministic policy with a higher return than the behavior policy.

\section{Related Work}

Offline RL, also known as batch RL, learns or evaluates a policy from a fixed batch of historical data~\citep{DBLP:journals/corr/abs-2005-01643}.
Overall, these algorithms rely on pessimism to 
mitigate the lack of feedback from the environment
and can be divided in two categories~\citep{Jin2021}:
those that constrain the final policy to stay close to the behavior policy~\citep{DBLP:conf/icml/LarocheTC19}, and those that penalize rare experiences~\citep{Petrik2016}.
Our method belongs to the first category.

Various extensions of SPIBB could be adapted for POMDPs, such as 
    soft-SPIBB~\citep{DBLP:conf/pkdd/NadjahiLC19,DBLP:conf/icaart/SchollDOU22},
    deep-SPIBB~\citep{Brandfonbrener2022}, and 
    factored-SPIBB~\citep{DBLP:conf/aaai/SimaoS19,DBLP:conf/ijcai/SimaoS19}.
SPI has also been studied without the behavior policy~\citep{DBLP:conf/atal/SimaoLC20} and for multi-objective~\citep{DBLP:conf/nips/SatijaTPL21} and non-stationary settings~\citep{DBLP:conf/nips/ChandakJTWT20}.

When the behavior policy is influenced by unobserved variables, we may come across confounding variables.
The problem of evaluating a policy offline was studied in this setting, for instance, assuming that observed and unobserved variables are decoupled~\citep{Tennenholtz2020}, or that the influence of the confounding variable on the behavior policy is limited~\citep{Namkoong2020}.
Since we assume that the behavior policy only depends on the observed history, we have no confounding variables.

\section{Conclusions}
We presented a new approach to safe policy improvement in POMDPs.
Our experiments show the applicability of the approach, even in cases where finite-history is not sufficient to obtain optimal results.
In the future, it would be interesting to relax Assumption~\ref{asm:1} to distance metrics~\citep{DBLP:conf/uai/FernsPP04,DBLP:conf/uai/FernsPP05} instead of exact bisimilarity.

\nocite{Weissman2003}

\section*{Acknowledgments}
We would like to thank Alberto Castellini, Alessandro Farinelli, Matthijs Spaan, and  Edoardo Zorzi for discussions on related topics.
We also thank Maris Galesloot for help with the implementation of the PO-UCT algorithm.
This research has been partially funded by NWO grants OCENW.KLEIN.187 (Provably Correct Policies for Uncertain Partially Observable Markov Decision Processes) and NWA.1160.18.238 (PrimaVera).

\bibliography{references}

\newpage
\onecolumn
\appendix

\newpage
\section{Proof of Theorem~\ref{th:spibb_history_mdp}}\label{appx:proof}

\paragraph{Theorem~\ref{th:spibb_history_mdp}} ($\zeta$-bound on history MDP)\textbf{.}\hspace{1em} 
\emph{Let $\Pib$ be the set of policies under the constraint of following $\pib$ when
$(\nza) \in \sU$.
Then, the policy $\pii$ computed by the SPIBB algorithm on the history MDP (Definition~\ref{def:hismdp}) is a $\zeta$-approximate safe policy improvement over the behavior policy $\pib$ with high probability $1-\delta$,
where:
\begin{equation}
    \zeta =
    \frac{4V_{\max}}{1-\gamma} \sqrt{\frac{2}{\Nmin}\log \frac{2 |\sH||A|2^{|Z|}}{\delta}} - \rho(\pii,\mlemdp)  + \rho(\pib,\mlemdp).
\label{eq:zetaSPIBB_POMCP_ha}
\end{equation}
}

\begin{proof}[Proof of theorem~\ref{th:spibb_history_mdp}]
First, we restate Theorem 2.1 by \citet{Weissman2003} that bounds the L1-error of an empirical probability distribution $\tilde{P}$ given a finite number of samples.

\begin{proposition*}[Theorem 2.1 by \citet{Weissman2003}]
Given $m$ i.i.d. random variables distributed according to $P$ and given an estimated probability distribution $\tilde{P}$ computed from those $m$ variables, we have
\[
    \Pr(\| P - \tilde{P}\|_1 \geq \epsilon)
    \leq (2^n-2)e^{-\nicefrac{m\epsilon^2}{2}}.
\]
\end{proposition*}

To bound the L1-error of the estimated transition function of given a $(h, a)$ pair of the history MDP, given $\#_{\D}(h,a)$ samples for $(h,a) \in \D$ we can apply the Theorem 2.1 \citep{Weissman2003} to obtain
\begin{align*}
    \Pr\left(\| T(\cdot \mid h, a) - \tilde{T}( \cdot \mid h,a) \|_1 \geq \epsilon\right) \\
    \quad \leq (2^{|Z|}-2)e^{-\nicefrac{\#_{\D}(h,a)\epsilon^2}{2}}.
\end{align*}
Choosing an appropriate $\Nmin$, we can do a union bound over all history-action pairs, we get that the probability of having an arbitrary history-action pair $(h, a)$ of which the L1-error is larger than $\epsilon$ is bounded by $\delta$:
\begin{align*}
\Pr&\left(\exists (h,a) \in \sH\times A: \| T(\cdot \mid h, a) - \tilde{T}(\cdot \mid h, a) \|_1 \geq \epsilon \right) \\
    & \quad \leq \sum_{h,a \in \sH \times A}Pr\left(\| T(\cdot \mid h,a) - \tilde{T}(\cdot \mid h,a) \|_1 \geq \epsilon\right) \\
    & \quad \leq |\sH|\,|A|   (2^{|Z|}-2)e^{-\nicefrac{\Nmin\epsilon^2}{2}} \\
    & \quad = \delta.
\end{align*}
Reordering the equation above, we get
\begin{equation}\label{eq:epsilon}
\epsilon = \sqrt{\frac{2}{\Nmin}\log \frac{2 |\sH||A|2^{|Z|}}{\delta}},
\end{equation}
and then
\begin{align*}
\Pr\left(\exists h,a \in \sH\times A. \| T(\cdot \mid h, a) - \tilde{T}( \cdot \mid h,a) \|_1 \geq \epsilon \right) \leq \delta.
\end{align*}

The last step of the proof is to replace the $\epsilon$ from Equation 36 of~\citep{DBLP:conf/icml/LarocheTC19} by the value from Equation~\eqref{eq:epsilon}, resulting in  
\begin{align*}
    &\zeta =
    \frac{4V_{\max}}{1-\gamma} \sqrt{\frac{2}{\Nmin}\log \frac{2 |\sH||A|2^{|Z|}}{\delta}} - \rho(\pii,\mlemdp)  + \rho(\pib, \mlemdp),
\end{align*}
as stated in the Theorem.
\end{proof}

\section{Experiments Details}\label{appx:experiments_description}

\paragraph{Further details on the problems.}
All experiments use a discount factor $\gamma = 0.95$, with at most $300$ steps in an episode.

\paragraph{Training the behavior policies.}
We train a Q-learning agent over $5\,000$ episodes with learning rate($\alpha$) and exploration rate~($\epsilon$) decaying exponentially after each episode.
\[
    \alpha_{i} \leftarrow \alpha_0 \exp(-\lambda * i),
\]
\[
    \epsilon_{i} \leftarrow \epsilon_0 \exp(-\lambda * i),
\]
where $i$ is the episode index, $\lambda$ is the decay rate, $\alpha_0$ and $\epsilon_0$ are the initial learning rate and initial exploration rate, respectively.

After training, we extract the behavior policy $\pib$ using a softmax function 
\[
\pib(a \mid s) = \frac{e^{-\tau Q(s, a)}}{\sum_{a'\in A}e^{-\tau Q(s, a')}},
\]
where $Q(s,a)$ is the final value of the state-action pair $s,a$ and $\tau$ a hyperparameter to control the randomness of the behavior policy.
As mentioned in the main paper, the state in this case is the history of the last $k$ observations.

\begin{table}[tbph]
    \centering
    \begin{tabular}{lcrrr} \toprule
         && CheeseMaze & Tiger & Voicemail  \\ 
         \midrule
         Initial exploration rate &$\epsilon_0$ &$0.500$&$0.500$&$0.500$ \\
         Initial learning rate &$\alpha_0$      &$1.000$&$1.000$&$1.000$ \\
         Decay rate &$\lambda$                   &$0.002$&$0.002$&$0.002$ \\
         Softmax temperature & $\tau$                 &$15.000$&$0.050$&$0.300$ \\
         \\
         \bottomrule
    \end{tabular}
    \caption{Hyperparameters to generate the behavior policies.}
    \label{tab:my_label}
\end{table}

\paragraph{Computing specifications.} All experiments were performed on a machine with a 4GHz Intel Core i9 CPU and 64Gb of memory, using a single core for each experiment.

\section{Supplementary Experimental Results}\label{appx:results}

\begin{figure*}[htbp]
    \centering
    \begin{minipage}{.87\textwidth}
    \includegraphics[width=.24\textwidth]{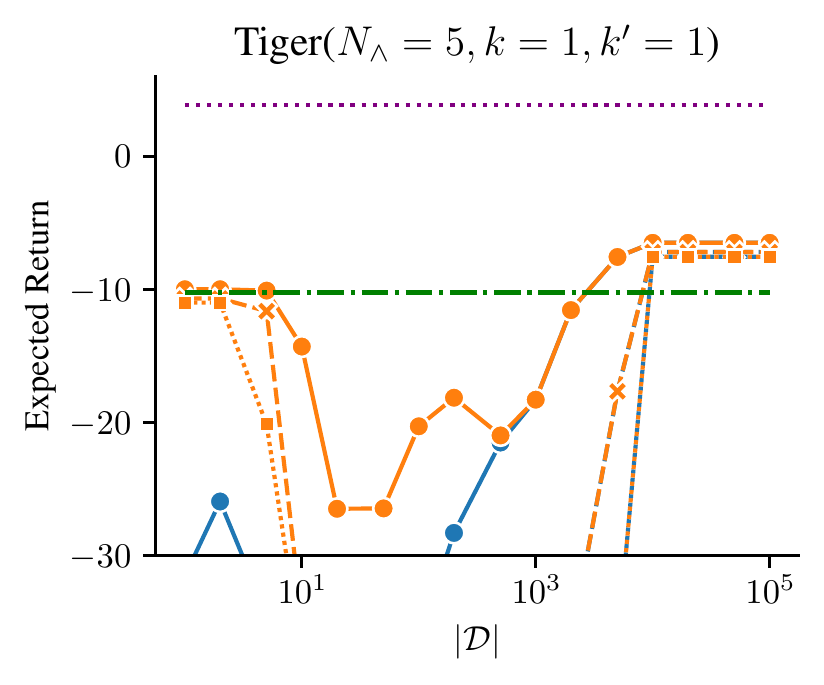}  
    \includegraphics[width=.24\textwidth]{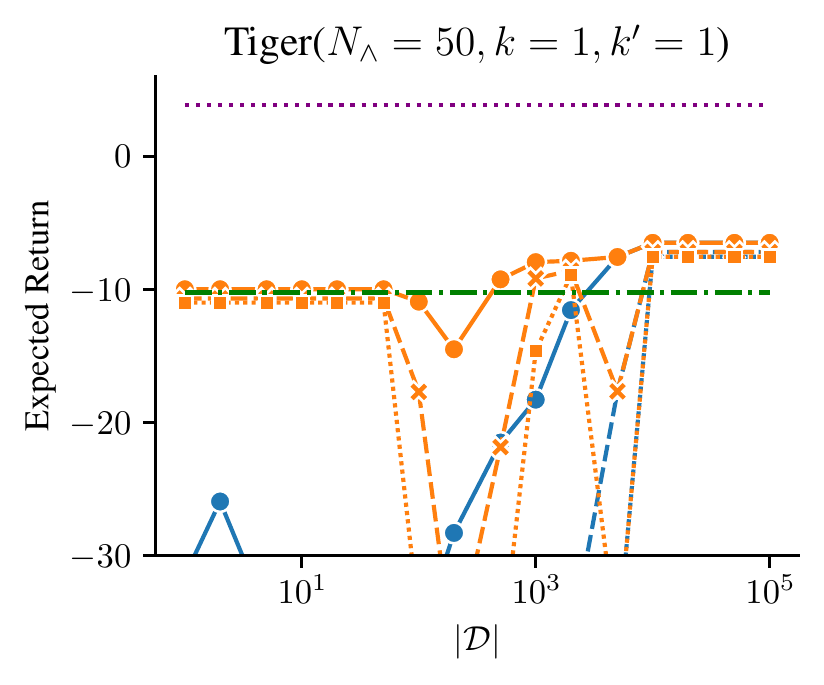} 
    \includegraphics[width=.24\textwidth]{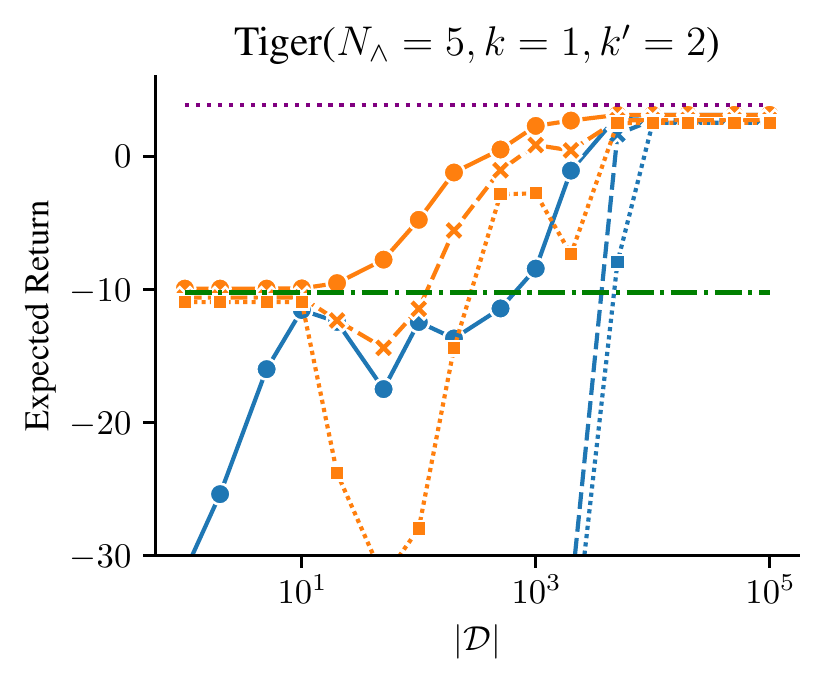} 
    \includegraphics[width=.24\textwidth]{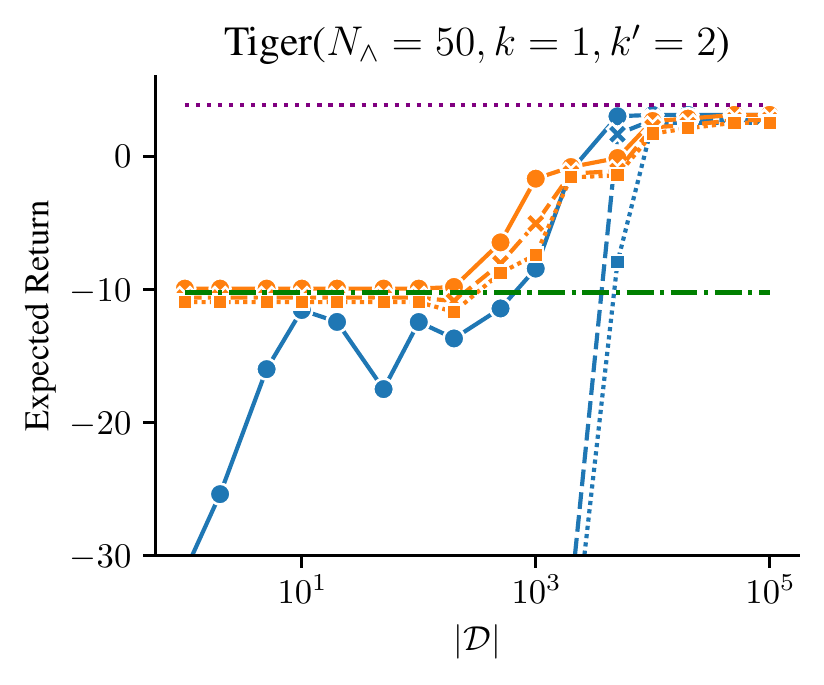}
    \end{minipage}
    \begin{minipage}{.12\textwidth}
    \includegraphics[width=\textwidth]{Tiger-v0_k-1/curve_legend} 
    \end{minipage}
    \caption{Policy improvement on the Tiger environment for datasets collected with a memoryless policy ($k=1$), varying the hyperparameters pairs column-wise: $(\Nmin = 5, k' = 1)$, $(\Nmin = 50, k' = 1)$,  $(\Nmin = 5, k' = 2)$, and $(\Nmin = 50, k' = 2)$. The plots show the mean (solid line), $10\%$-CVaR (dashed line) and $1\%$-CVaR (dotted line). The performance of the behavior policy is shown in green (dash-dotted line).}
    \label{fig:performance_tiger_memoryless}
\end{figure*}

\begin{figure*}[tbp]
    \centering
    
    \begin{minipage}{.92\textwidth}
    \includegraphics[width=.32\textwidth]{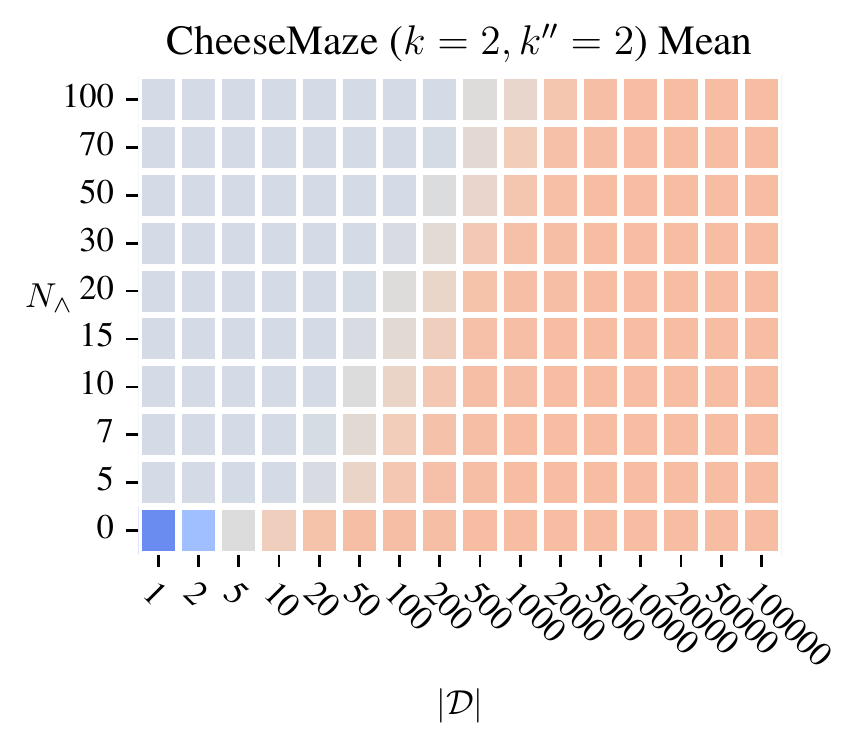}
    \includegraphics[width=.32\textwidth]{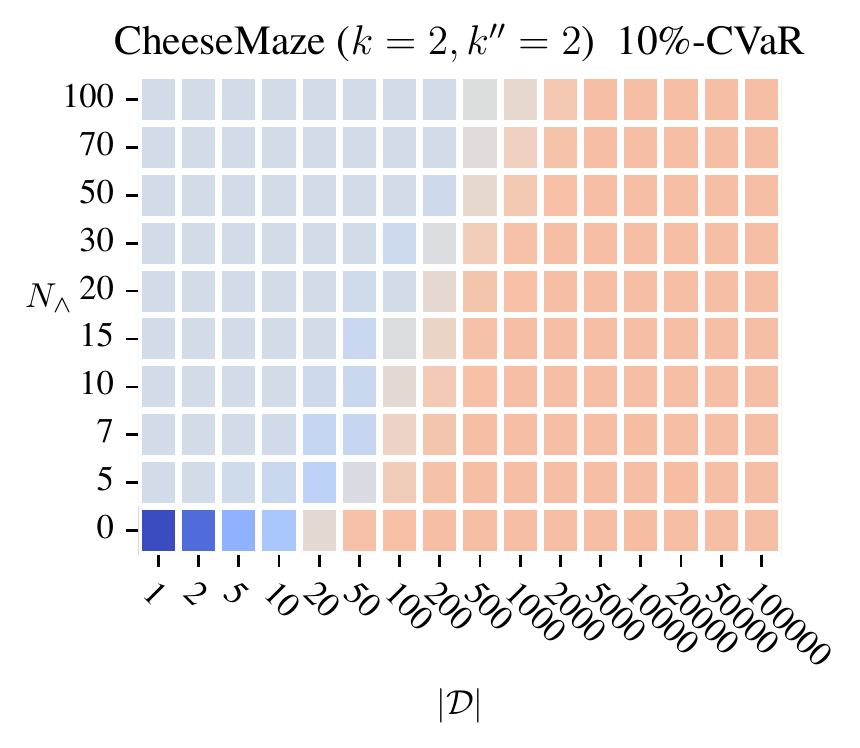}
    \includegraphics[width=.32\textwidth]{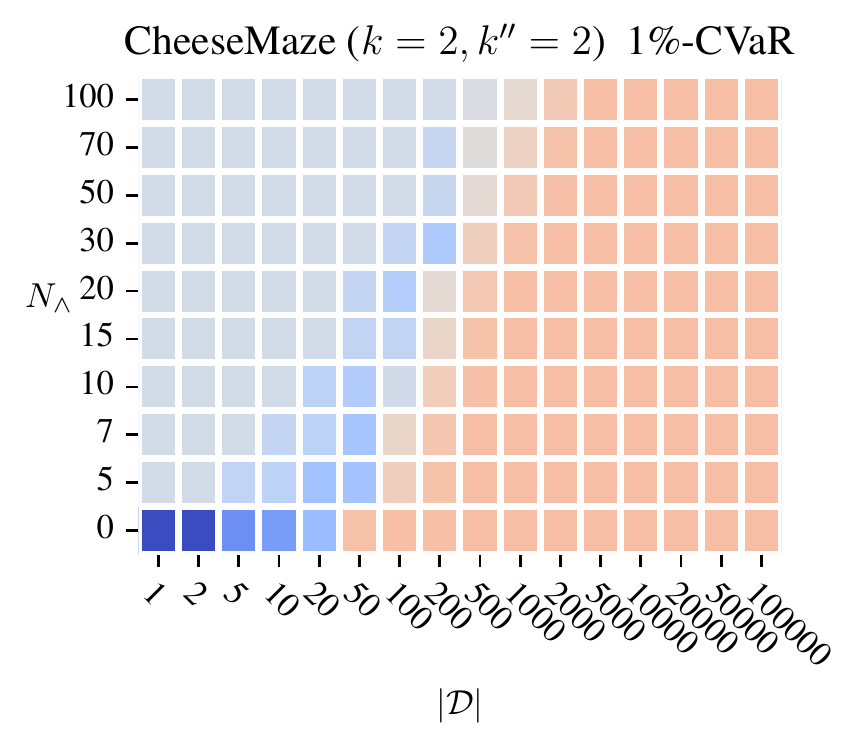}\\
    \includegraphics[width=.32\textwidth]{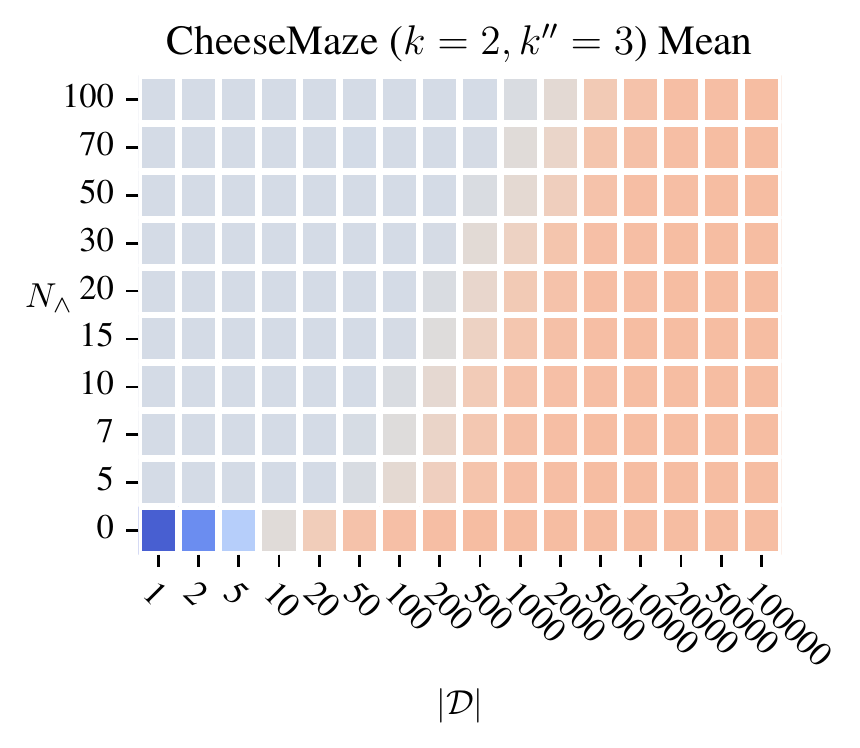}
    \includegraphics[width=.32\textwidth]{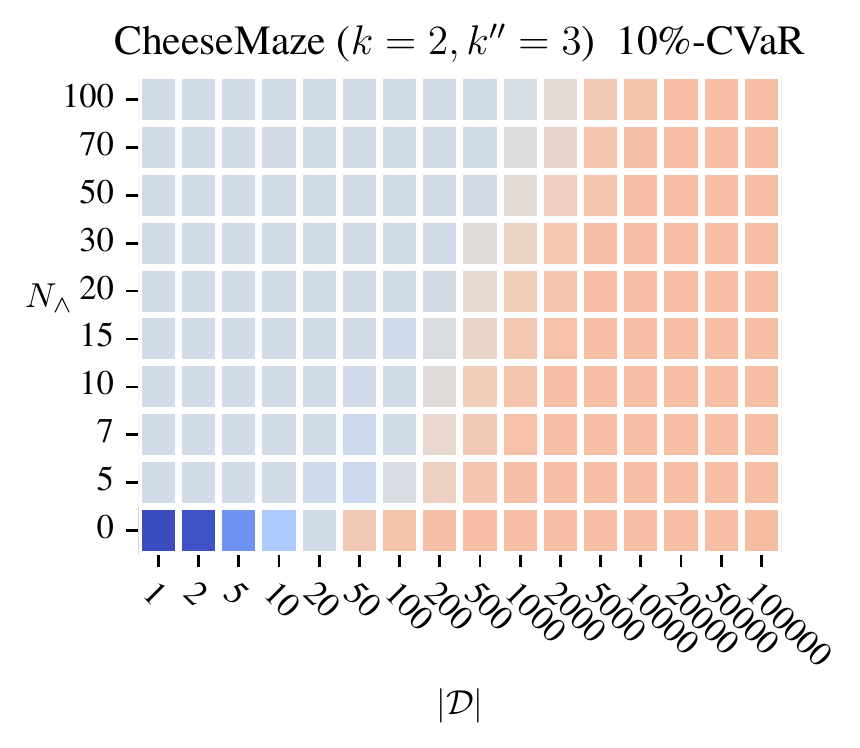}
    \includegraphics[width=.32\textwidth]{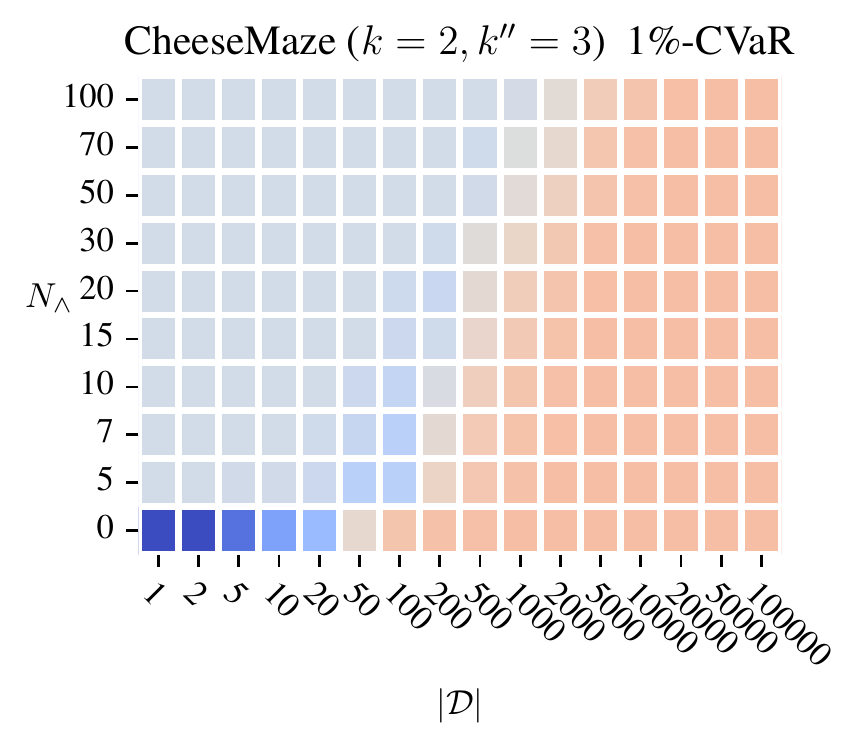}
    \end{minipage}
    \begin{minipage}{.07\textwidth}
    \includegraphics[width=\textwidth]{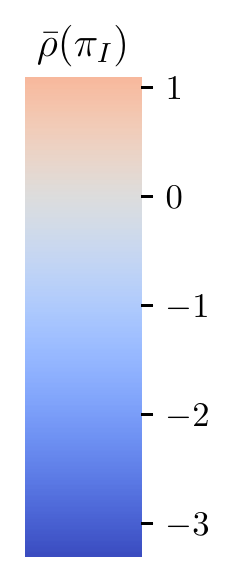}
    \end{minipage}

    \caption{
    Normalized results ($\bar{\rho}(\pii)$) on the Maze environment ($k=2$). \heatmapdescription{}
    }
    \label{fig:maze}
\end{figure*}

\begin{figure*}[tbp]
    \centering
    \begin{minipage}{.92\textwidth}
    \includegraphics[width=.32\textwidth]{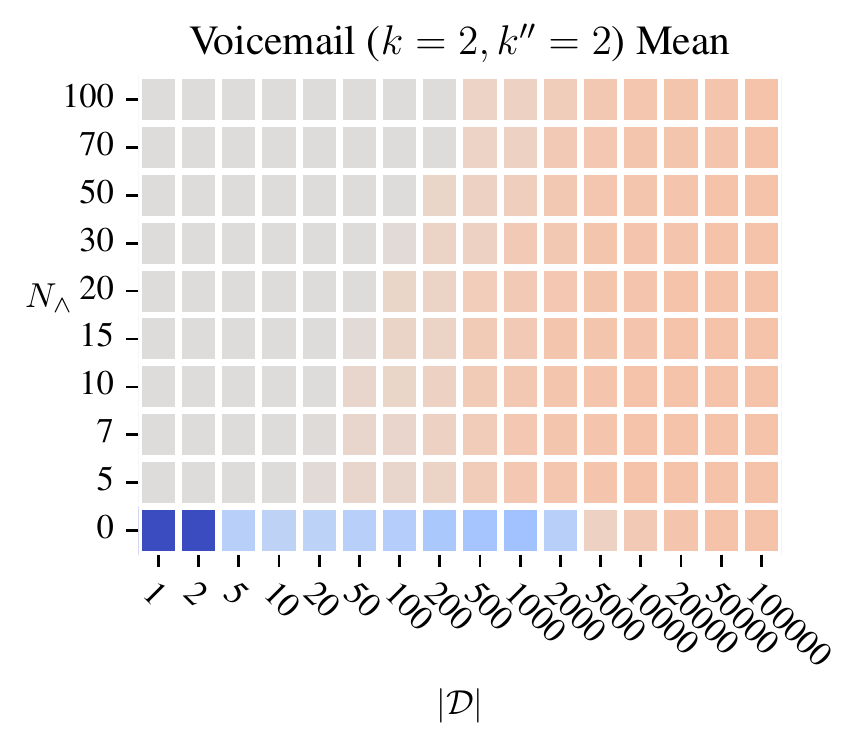}
    \includegraphics[width=.32\textwidth]{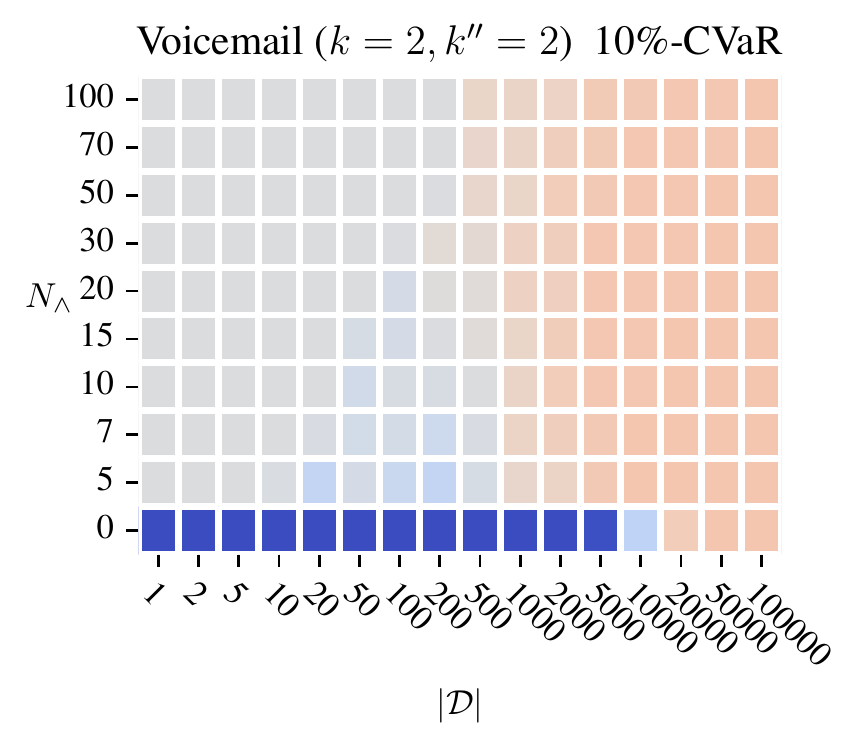}
    \includegraphics[width=.32\textwidth]{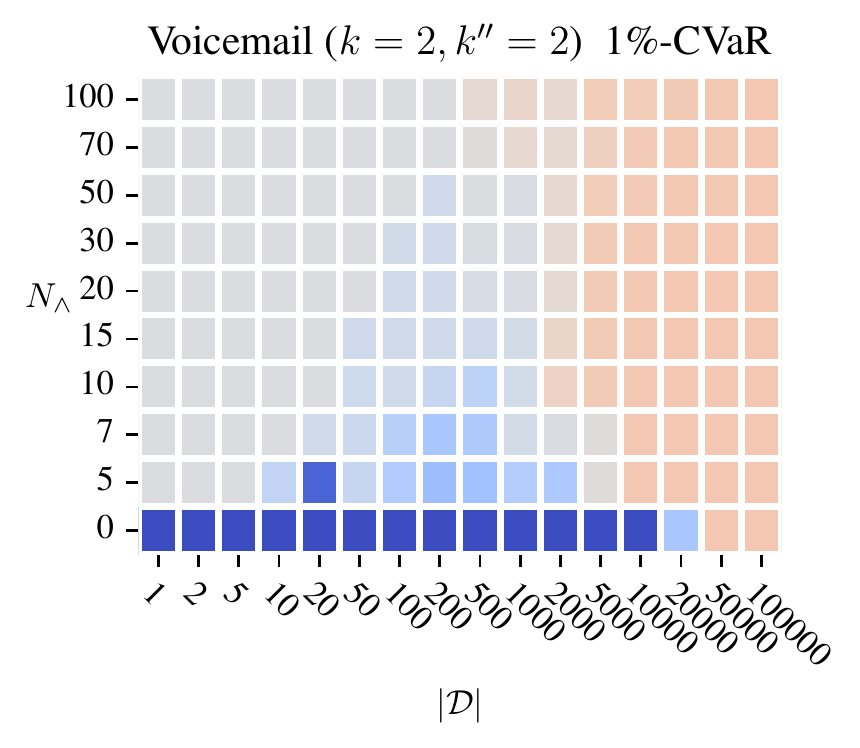}\\
    \includegraphics[width=.32\textwidth]{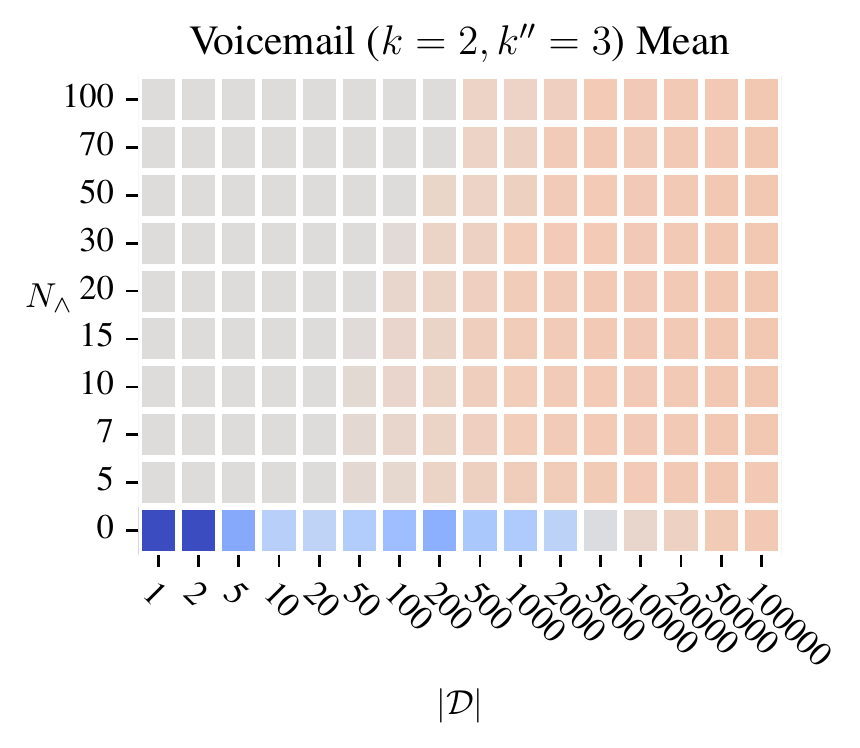}
    \includegraphics[width=.32\textwidth]{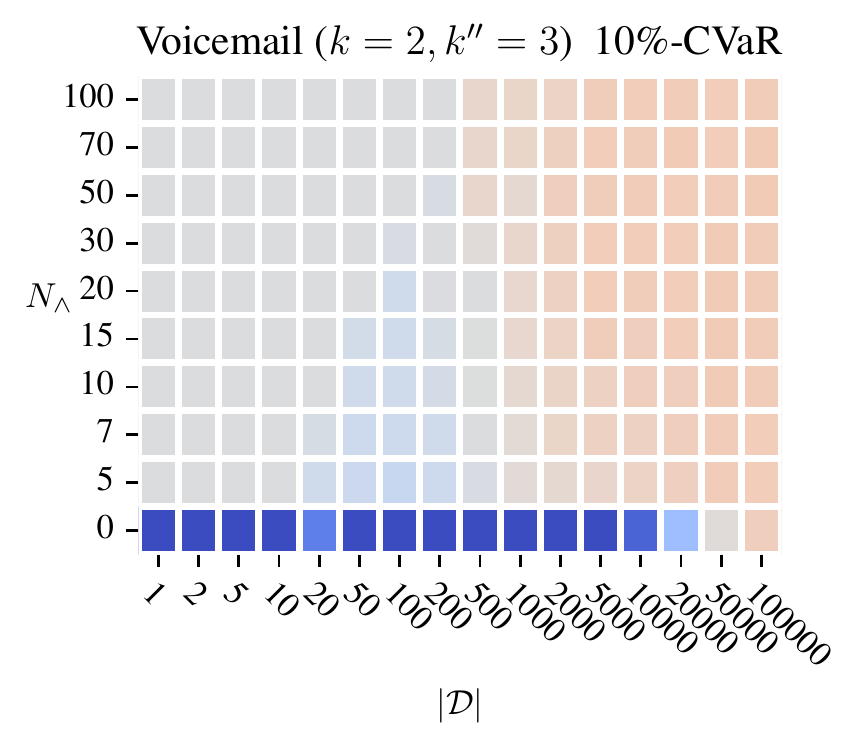}
    \includegraphics[width=.32\textwidth]{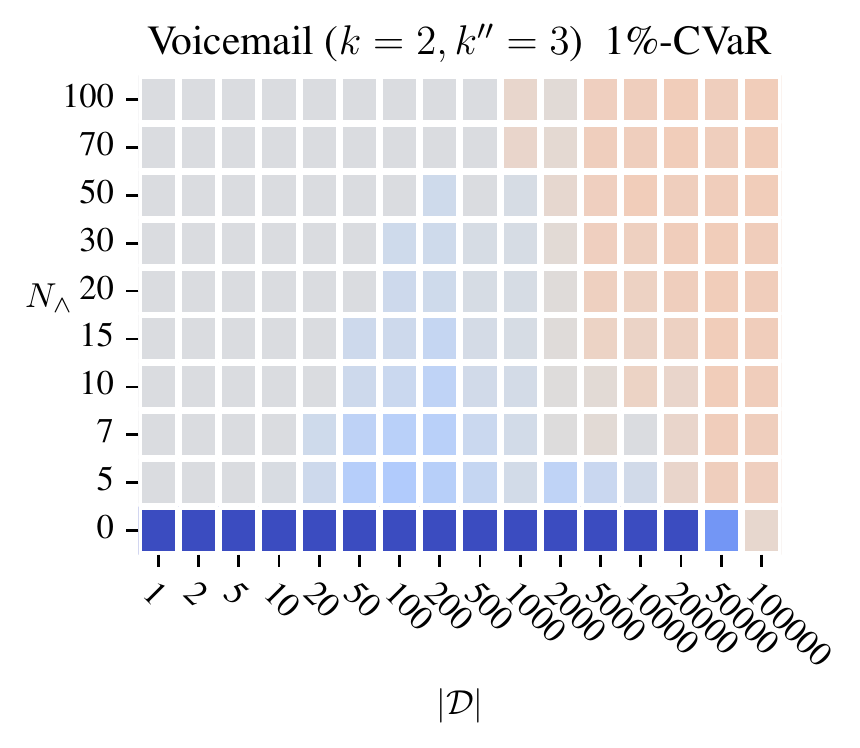}
    \end{minipage}
    \begin{minipage}{.07\textwidth}
    \includegraphics[width=\textwidth]{Voicemail-v0_k-2/heatmap_legend.pdf}
    \end{minipage}
    \caption{
    Normalized result on the Voicemail environment ($k=2$). \heatmapdescription{}
    }
    \label{fig:voicemail_heatmap_2}
\end{figure*}

\begin{figure*}[tbp]
    \centering
    \begin{minipage}{.92\textwidth}
    \includegraphics[width=.32\textwidth]{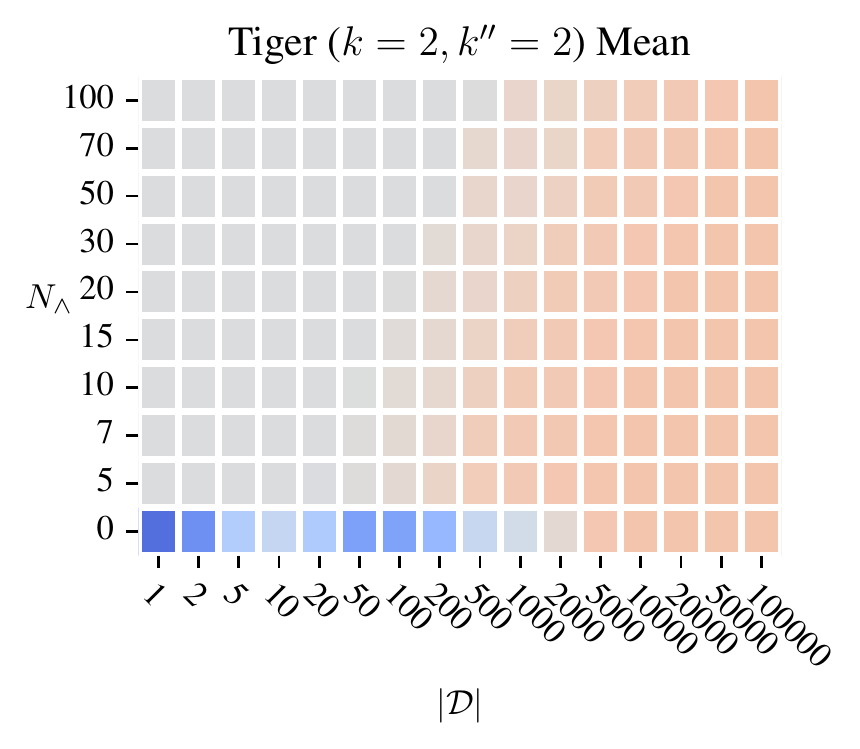}   
    \includegraphics[width=.32\textwidth]{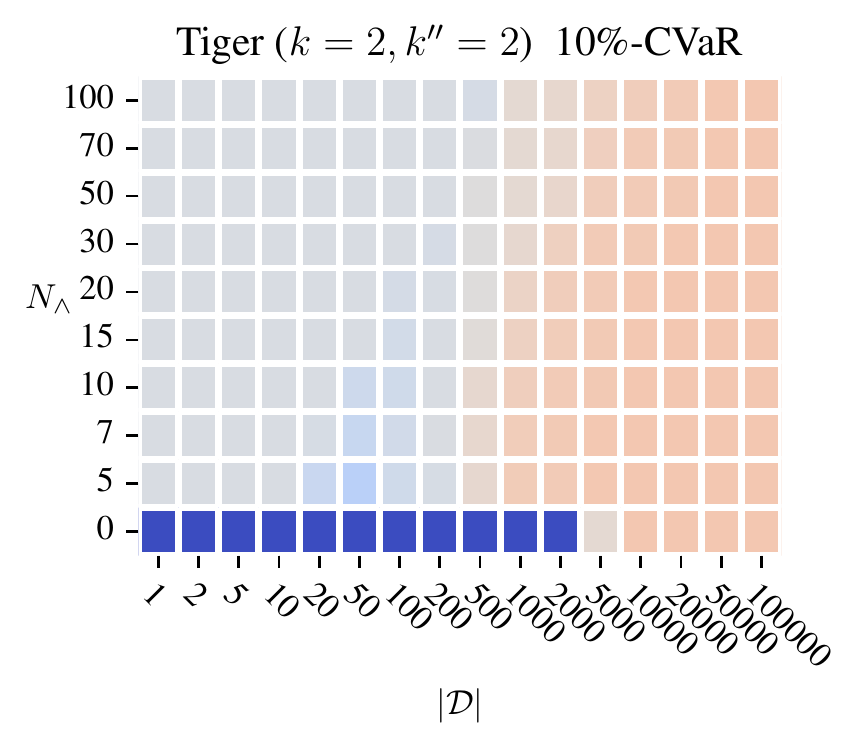} \includegraphics[width=.32\textwidth]{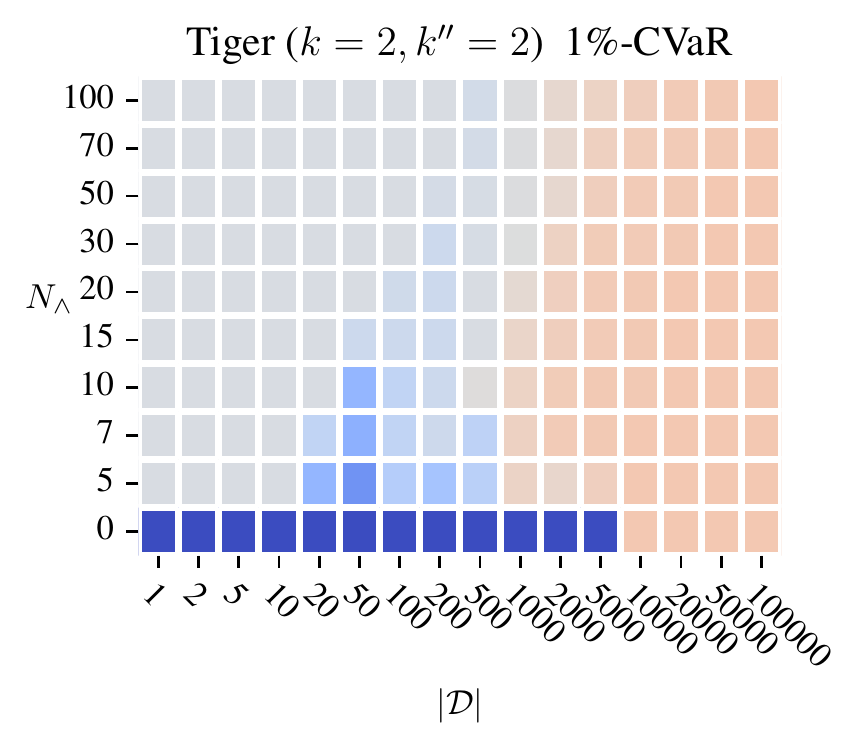}\\
    \includegraphics[width=.32\textwidth]{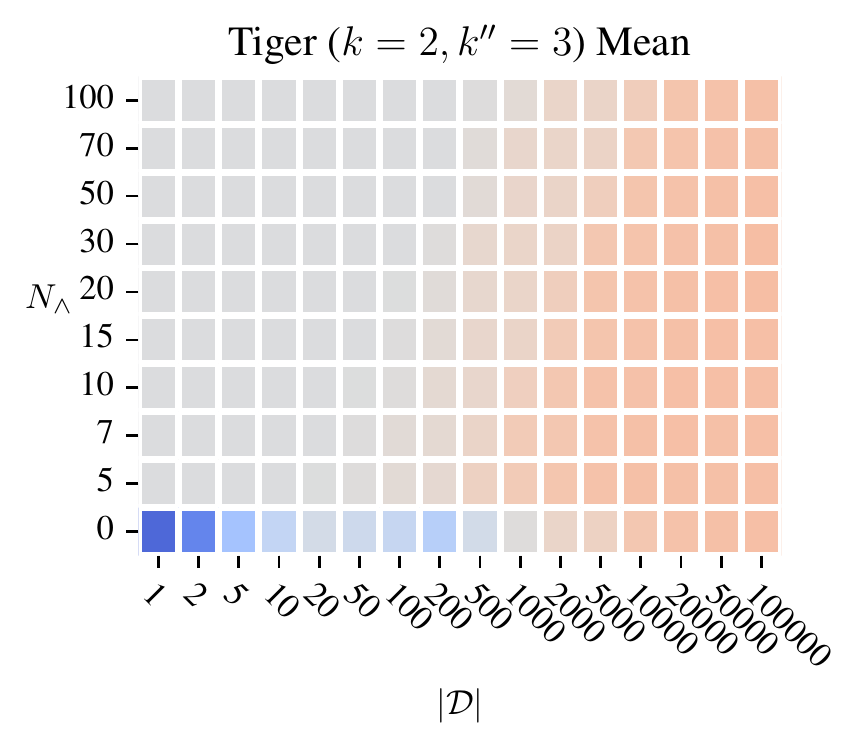}
    \includegraphics[width=.32\textwidth]{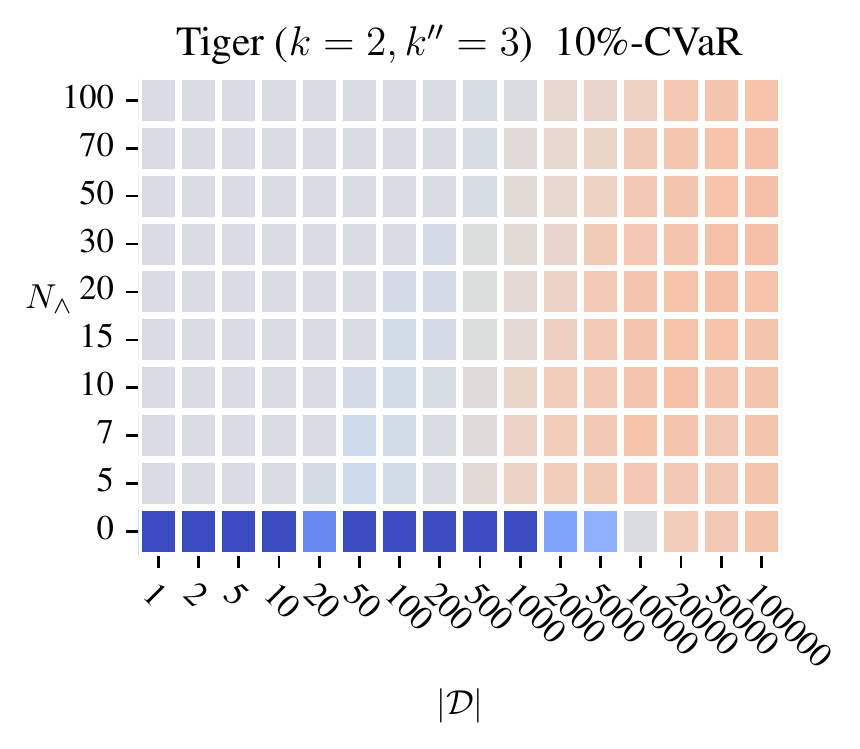}
    \includegraphics[width=.32\textwidth]{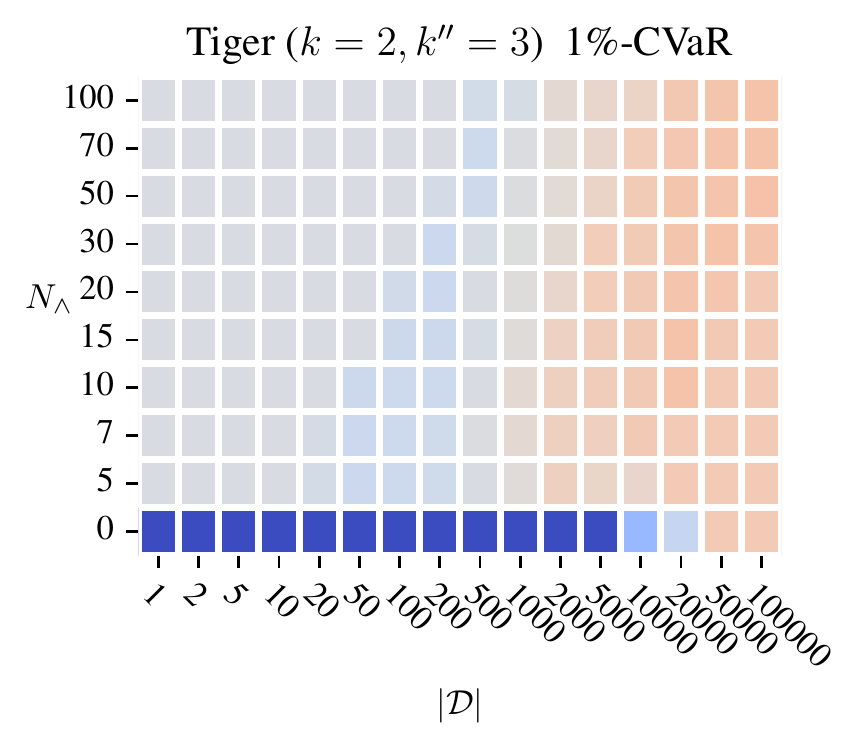}
    \end{minipage}
    \begin{minipage}{.07\textwidth}
    \includegraphics[width=\textwidth]{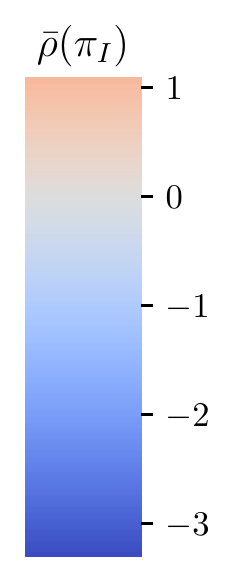}
    \end{minipage}

    \caption{Normalized results ($\bar{\rho}(\pii)$) on the Tiger environment ($k=2$). \heatmapdescription{}
    }
    \label{fig:tiger_heatmap_2}
\end{figure*}

\begin{figure*}[tbp]
    \centering
    \begin{minipage}{.92\textwidth}
    \includegraphics[width=.32\textwidth]{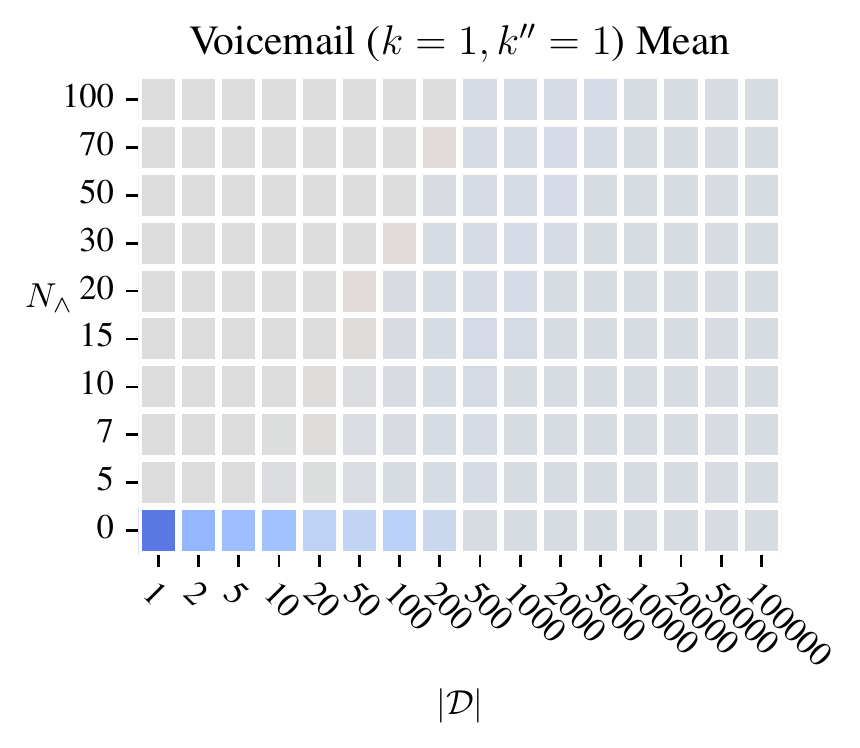}   
    \includegraphics[width=.32\textwidth]{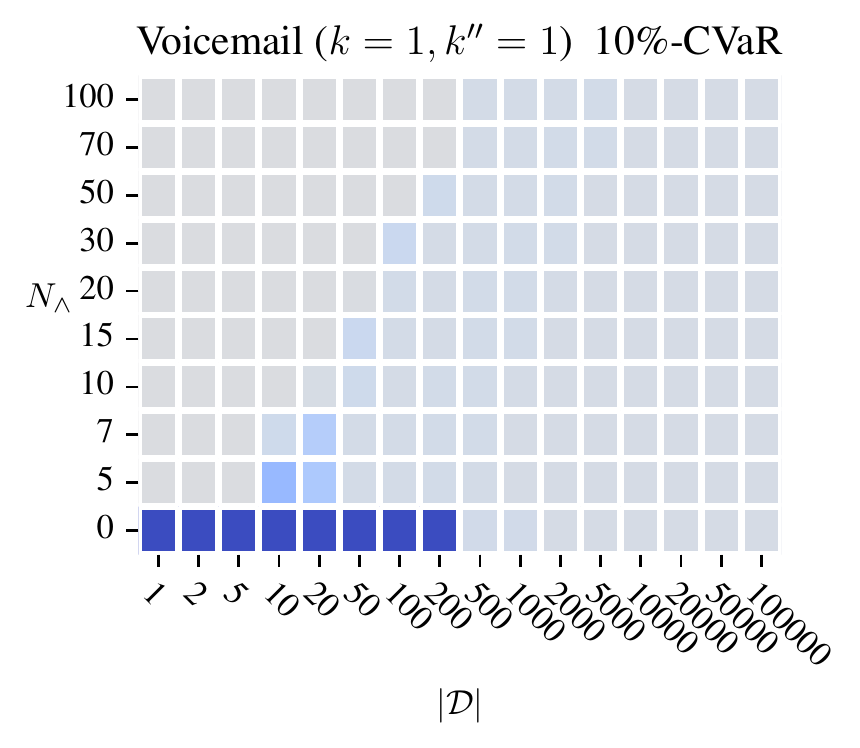} \includegraphics[width=.32\textwidth]{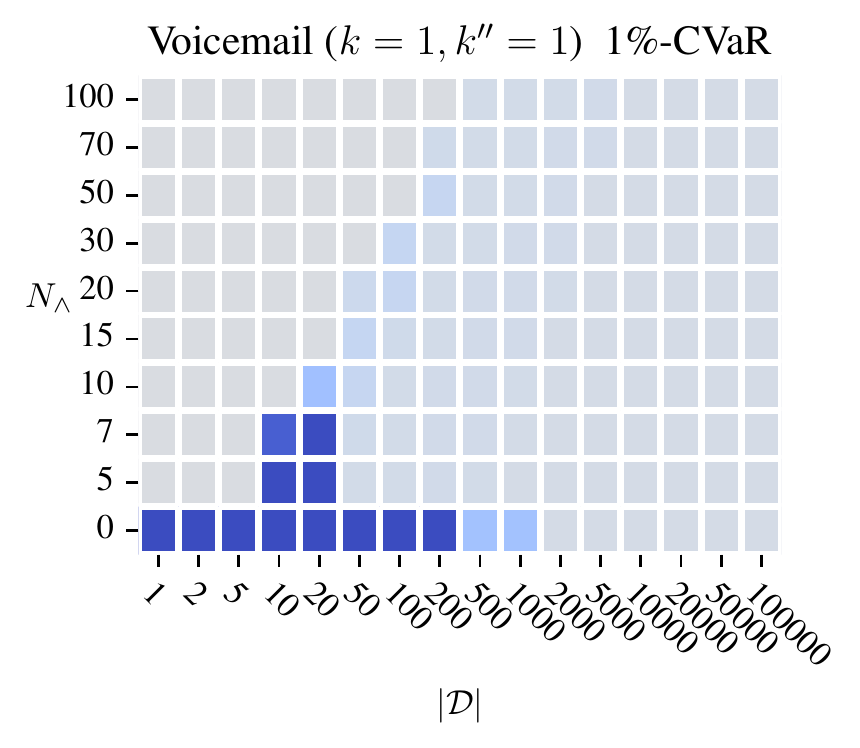}\\
    \includegraphics[width=.32\textwidth]{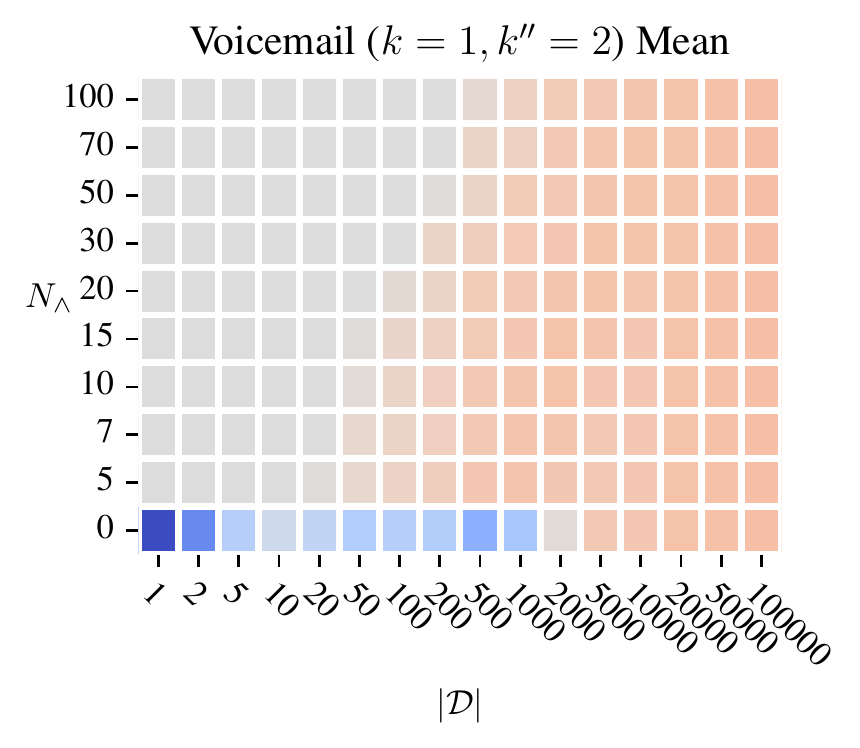}
    \includegraphics[width=.32\textwidth]{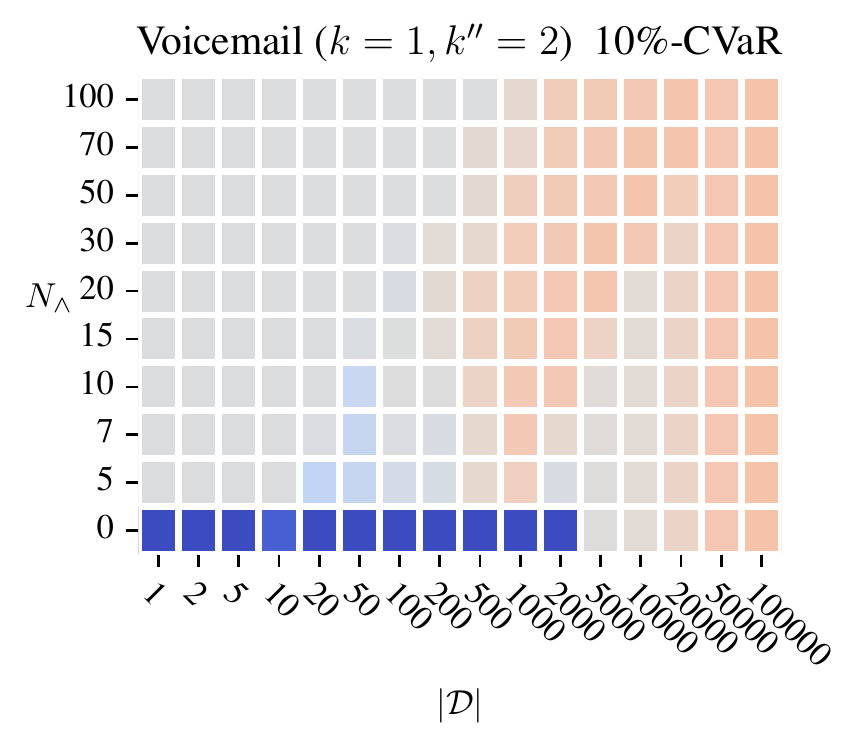}
    \includegraphics[width=.32\textwidth]{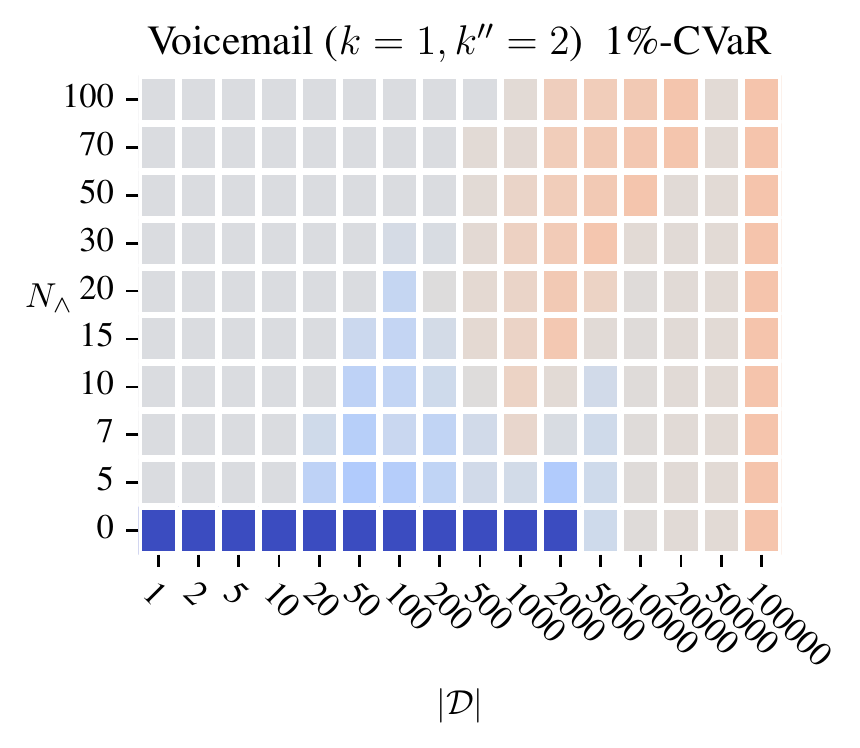}
    \end{minipage}
    \begin{minipage}{.07\textwidth}
    \includegraphics[width=\textwidth]{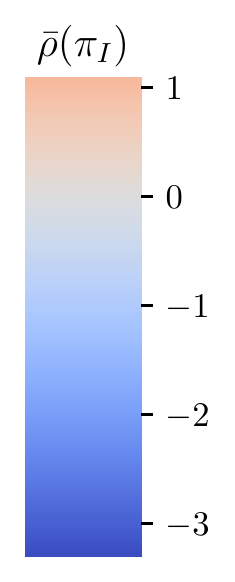}
    \end{minipage}

    \caption{Normalized results ($\bar{\rho}(\pii)$) on the Voicemail environment ($k=1$). \heatmapdescription{}
    }
    \label{fig:voicemail_heatmap_1}
\end{figure*}

\end{document}

%% file: packages.tex
\usepackage[]{aaai23}  %
\usepackage{times}  %
\usepackage{helvet}  %
\usepackage{courier}  %
\usepackage[hyphens]{url}  %
\usepackage{graphicx} %
\usepackage{natbib}  %
\usepackage{caption} %
\usepackage{algorithm}
\usepackage{algorithmic}

\usepackage{newfloat}
\usepackage{listings}
\DeclareCaptionStyle{ruled}{labelfont=normalfont,labelsep=colon,strut=off} %
\floatstyle{ruled}
\newfloat{listing}{tb}{lst}{}
\floatname{listing}{Listing}
\usepackage[utf8]{inputenc}
\usepackage{amsfonts}
\usepackage{amsmath}
\usepackage{amssymb}
\usepackage{amsthm}
\usepackage{booktabs}
\usepackage{enumitem}
\setlist[enumerate,1]{label=\textit{\roman*})}
\usepackage{fontawesome5}
\usepackage{graphicx}
\usepackage{natbib}
\usepackage{nicefrac}

\usepackage{comment}

%% file: macros.tex
\newcommand{\sH}{\mathcal{H}}
\newcommand{\D}{\mathcal{D}}
\newcommand{\sU}{\mathcal{U}}
\newcommand{\Dist}{\Delta}
\newcommand{\R}{\mathbb{R}}
\newcommand{\E}{\mathbb{E}}

\newtheorem{theorem}{Theorem}
\newtheorem{definition}{Definition}
\newtheorem{assumption}{Assumption}
\newtheorem*{assumption*}{Assumption}

\newtheorem*{proposition*}{Proposition}
\newtheorem{corollary}{Corollary}
\newtheorem*{corollary*}{Corollary}

\newcommand{\ie}{\emph{i.e.}}

\newcommand{\pomdp}{\mathcal{M}}

\newcommand{\actionmap}{\ensuremath{\psi}}
\newcommand{\memupdate}{\ensuremath{\eta}}
\newcommand{\sN}{\ensuremath{\mathcal{N}}}

\newcommand{\pib}{\ensuremath{\pi_\beta}}
\newcommand{\Pib}{\ensuremath{\Pi_\beta}}

\newcommand{\pii}{\ensuremath{\pi_I}}

\newcommand{\nz}{\ensuremath{\langle n, z \rangle}}
\newcommand{\nza}{\ensuremath{\nz, a}}
\newcommand{\nzprime}{\ensuremath{\langle n', z' \rangle}}
\newcommand{\NZ}{\ensuremath{\sN\!\!\times\!\!Z}}

\newcommand{\mlemdp}{\ensuremath{\tilde{M}}}

\newcommand{\Nmin}{\ensuremath{N_\wedge}}

%% file: images/policy_update_offline_pomdp.pdf_tex
\begingroup%
  \makeatletter%
  \providecommand\color[2][]{%
    \errmessage{(Inkscape) Color is used for the text in Inkscape, but the package 'color.sty' is not loaded}%
    \renewcommand\color[2][]{}%
  }%
  \providecommand\transparent[1]{%
    \errmessage{(Inkscape) Transparency is used (non-zero) for the text in Inkscape, but the package 'transparent.sty' is not loaded}%
    \renewcommand\transparent[1]{}%
  }%
  \providecommand\rotatebox[2]{#2}%
  \newcommand*\fsize{\dimexpr\f@size pt\relax}%
  \newcommand*\lineheight[1]{\fontsize{\fsize}{#1\fsize}\selectfont}%
  \ifx\svgwidth\undefined%
    \setlength{\unitlength}{479.04956055bp}%
    \ifx\svgscale\undefined%
      \relax%
    \else%
      \setlength{\unitlength}{\unitlength * \real{\svgscale}}%
    \fi%
  \else%
    \setlength{\unitlength}{\svgwidth}%
  \fi%
  \global\let\svgwidth\undefined%
  \global\let\svgscale\undefined%
  \makeatother%
  \begin{picture}(1,0.47579171)%
    \lineheight{1}%
    \setlength\tabcolsep{0pt}%
    \put(0.15845704,0.27716848){\makebox(0,0)[t]{\lineheight{1.25}\smash{\begin{tabular}[t]{c}$z,r$\end{tabular}}}}%
    \put(0,0){\includegraphics[width=\unitlength,page=1]{policy_update_offline_pomdp.pdf}}%
    \put(0.49520906,0.14602561){\makebox(0,0)[t]{\lineheight{1.25}\smash{\begin{tabular}[t]{c}learning\end{tabular}}}}%
    \put(0,0){\includegraphics[width=\unitlength,page=2]{policy_update_offline_pomdp.pdf}}%
    \put(0.84386127,0.3587104){\makebox(0,0)[t]{\lineheight{1.25}\smash{\begin{tabular}[t]{c}deployment\end{tabular}}}}%
    \put(0,0){\includegraphics[width=\unitlength,page=3]{policy_update_offline_pomdp.pdf}}%
    \put(0.22394446,0.20017868){\makebox(0,0)[t]{\lineheight{1.25}\smash{\begin{tabular}[t]{c}\huge$\pi_\beta$\end{tabular}}}}%
    \put(0.15845714,0.12054097){\makebox(0,0)[t]{\lineheight{1.25}\smash{\begin{tabular}[t]{c}$a$\end{tabular}}}}%
    \put(0.1561198,0.34618559){\makebox(0,0)[t]{\lineheight{1.25}\smash{\begin{tabular}[t]{c}data collection\end{tabular}}}}%
    \put(0.33132452,0.43383986){\makebox(0,0)[t]{\lineheight{1.25}\smash{\begin{tabular}[t]{c}$[(z_i,a_i,r_i,z_i')]$\end{tabular}}}}%
    \put(0.49387377,0.37056427){\makebox(0,0)[t]{\lineheight{1.25}\smash{\begin{tabular}[t]{c}buffer\end{tabular}}}}%
    \put(0.84719802,0.27716848){\makebox(0,0)[t]{\lineheight{1.25}\smash{\begin{tabular}[t]{c}$z,r$\end{tabular}}}}%
    \put(0.84719802,0.12054097){\makebox(0,0)[t]{\lineheight{1.25}\smash{\begin{tabular}[t]{c}$a$\end{tabular}}}}%
    \put(0.91228265,0.19790013){\makebox(0,0)[t]{\lineheight{1.25}\smash{\begin{tabular}[t]{c}\huge$\pi_I$\end{tabular}}}}%
    \put(0.49610492,0.08441081){\makebox(0,0)[t]{\lineheight{1.25}\smash{\begin{tabular}[t]{c}\huge$\pi_I$\end{tabular}}}}%
    \put(0.67098007,0.01668194){\makebox(0,0)[t]{\lineheight{1.25}\smash{\begin{tabular}[t]{c}$\pi_I$\end{tabular}}}}%
    \put(0.49592352,0.30165632){\makebox(0,0)[t]{\lineheight{1.25}\smash{\begin{tabular}[t]{c}\huge$\mathcal{D}$\end{tabular}}}}%
  \end{picture}%
\endgroup%